\documentclass{article}

\PassOptionsToPackage{numbers,compress}{natbib}
\usepackage[main,preprint]{neurips_2026}
\renewcommand{\acksection}{\section*{Acknowledgments}}

\usepackage[utf8]{inputenc}
\usepackage[T1]{fontenc}
\usepackage[hypertexnames=false]{hyperref}
\usepackage{url}
\usepackage{booktabs}
\usepackage{amsfonts}
\usepackage{nicefrac}
\usepackage{microtype}
\usepackage{multirow}
\usepackage{graphicx}
\usepackage[table]{xcolor}
\usepackage{array}
\usepackage{calc}
\usepackage{amsmath}
\usepackage{amssymb}
\usepackage{mathtools}
\usepackage{amsthm}
\usepackage{bm}
\usepackage{algorithm}
\usepackage{algpseudocode}

\newcolumntype{L}[1]{>{\raggedright\arraybackslash}p{#1}}

\definecolor{ourblue}{RGB}{235,244,255}

\theoremstyle{plain}
\newtheorem{theorem}{Theorem}[section]
\newtheorem{proposition}[theorem]{Proposition}
\newtheorem{lemma}[theorem]{Lemma}
\newtheorem{corollary}[theorem]{Corollary}
\theoremstyle{definition}

\theoremstyle{remark}

\def\x{{\mathbf x}}
\def\y{{\mathbf y}}

\def\f{{\mathbf{f}}}
\def\u{{\mathbf u}}

\def\h{\mathbf{h}}

\def\w{\mathbf{w}}

\def\vf{{\mathbf F}}
\def\vu{{\mathbf U}}

\title{Continuous-depth Deep Gaussian Processes}
\author{%
  \parbox{\dimexpr\textwidth-2\tabcolsep\relax}{\centering
    Ying Li$^{1}$ \quad Sanyou Wu$^{2}$ \quad Zi Yang$^{3}$\\
    Qiaochu Xu$^{3}$ \quad Yuhao Liu$^{4}$ \quad Zhidi Lin$^{5}$\\
    Michael Minyi Zhang$^{3}$ \quad Petar Djuric$^{6}$ \quad Alexandre H. Thiery$^{1}$\thanks{Corresponding author: \texttt{a.h.thiery@nus.edu.sg}.}\\[5pt]
    {\normalfont\fontsize{9}{11}\selectfont
    $^{1}$National University of Singapore \quad $^{2}$Columbia University\\
    $^{3}$The University of Hong Kong \quad $^{4}$Capital One\\
    $^{5}$The Education University of Hong Kong \quad $^{6}$Stony Brook University
    }
  }
}
\date{}

\begin{document}
\maketitle

\begin{abstract}
Deep Gaussian processes (DGPs) are appealing Bayesian regression models, but their standard discrete-depth parameterization often becomes harder to optimize as depth grows. We argue that depth is better treated as a latent continuous evolution than as a long discrete stack. We therefore propose a continuous-depth formulation in which the latent representation follows a stochastic differential equation with a depth-indexed drift family equipped with a GP-based prior. This reframes the core design choice as the prior placed on that family, rather than adopting a purely neural continuous-depth parameterization
as in neural ODE or neural SDE models. We study two instantiations of this idea: CDGP, a direct GP prior over the state--depth drift, and FlowDGP, a flow-evolved prior that captures richer depth dependence at lower training cost. To
motivate the redesign, we analyze how the single-sample DGP objective can become increasingly sensitive under repeated layer composition, while residual DGP is a useful first repair that still remains a finite discrete construction. We also give pathwise sensitivity analysis for
the continuous-depth formulation. Experiments on synthetic, benchmark, and large-scale regression tasks show that residual DGP improves over plain DGP, while CDGP and FlowDGP provide a stronger overall story across the regimes we study.
\end{abstract}

\section{Introduction}

\label{sec:intro}

Deep Gaussian processes (DGPs) promise hierarchical probabilistic
regression, but in practice they become difficult to optimize as depth
grows. This is exactly the regime where a DGP should be most useful, yet the standard discrete stack often becomes least reliable there. The central question is therefore not only how to stabilize a deep composition, but how depth itself should be parameterized in a
probabilistic latent-function model.

Empirically, deeper DGPs often exhibit pathological behavior, with
additional layers degrading performance
\citep{duvenaud2014deep, dunlop2018howdeep}. As a result, DGPs are
usually kept shallow, which undercuts their original modeling appeal.
This practical difficulty is consistent with a simple sensitivity view
of discrete composition: repeated layerwise transformations can amplify
parameter perturbations and make optimization progressively more
brittle.

Prior work clarifies parts of this picture without resolving it.
Theoretical studies identified depth pathologies in deep compositions
\citep{duvenaud2014deep, dunlop2018howdeep}, while residual DGP variants
\citep{wyrwalresidual} introduced the most direct repair by keeping an
identity path. Residual DGP is therefore the right place to start, but
it still leaves depth represented as a finite stack of separately
parameterized layers. Other successful
alternatives, such as deep kernel learning \citep{wilson2016deep} or
thin DGPs \citep{de2023thin}, can improve predictive performance, and
related continuous-depth constructions such as differential GP flows
\citep{hegde2019deep} show that GP-based dynamics can be useful. Our
focus is narrower: preserve a GP-native latent-function view of
depth while changing the way depth itself is parameterized.

Plain discrete-depth DGPs admit an optimization pathology: repeated
layer composition can make the single-sample training objective
increasingly sensitive to parameter perturbations. Residual DGP is the
natural first fix because it weakens that amplification through an
identity shortcut. But it does not change the underlying parameterization of
depth. Once depth is treated as continuous latent evolution, the most
natural non-GP alternatives are neural ordinary differential equations (ODE) and neural stochastic differential equations (SDE) models. We
study the complementary design in which the deep stack is replaced by a
latent continuous-depth system whose drift family retains a GP-based
prior. Once this latent SDE is in place, the central design choice is
no longer the number of layers, but the prior placed on that drift
family. This leads to two concrete models: CDGP, which uses a direct
separable GP prior on the drift family, and FlowDGP, which uses a GP
prior on a base drift together with flow evolution across depth.

Our contributions are fourfold. First, we propose a continuous-depth
methodology for DGP regression in which the latent representation
evolves as an SDE with a depth-indexed drift family equipped with a
GP-based prior. Second, we study two priors on that drift, CDGP and
FlowDGP, within a shared sparse variational inference framework. Third,
we identify the repeated Jacobian products that control sensitivity in
discrete DGPs and derive a finite-horizon pathwise sensitivity bound for
the continuous-depth formulation. The stronger contractive corollary is
a specialized auxiliary result, not a property used to design our
models. Fourth, we show empirically that residual DGP is a
meaningful first repair of plain DGP, and that against both
discrete-depth baselines and the natural continuous-depth neural
alternatives, CDGP and FlowDGP are more stable or more competitive
across toy, benchmark, and scalable regression settings.
Continuous depth is a better way to represent depth in deep GP
regression than a long discrete stack, especially when the drift family
retains GP-based structure rather than being parameterized only by
neural dynamics.

\vspace{-0.05in}
\section{Background}
\vspace{-0.05in}

\label{sec:background}

We review only the ingredients needed for the continuous-depth
reformulation: GPs as priors over latent functions, DGPs as discrete
compositions of those functions. 

\textbf{Gaussian Process (GP).} A GP \citep{rasmussen2006gaussian}
defines a prior over scalar-valued functions
$f:\mathbb{R}^D \to \mathbb{R}$, written
$f(\mathbf{x}) \sim \mathcal{GP}(m(\mathbf{x}),
k(\mathbf{x},\mathbf{x}'))$, where $m(\mathbf{x})$ and
$k(\mathbf{x},\mathbf{x}')$ are the mean and covariance functions.
Unless otherwise stated, we assume zero mean and write
$k(\mathbf{x},\mathbf{x}')$ simply as $k$ when no ambiguity arises.
For vector-valued mappings $f:\mathbb{R}^D \to \mathbb{R}^P$, we model
each output dimension independently: $f_p(\cdot) \sim \mathcal{GP}(0,
k)$ for $p=1,\dots,P$. Thus $f$ is represented by $P$
scalar-valued GPs sharing the same kernel.

\textbf{Deep Gaussian Process (DGP).} A DGP \citep{damianou2013deep}
defines a hierarchy of latent representations through a composition of
GP mappings. Let $\mathbf{h}^{(\ell)} \in \mathbb{R}^{D_\ell}$ denote
the latent representation at layer $\ell$, for $\ell=1,\dots,L$,
with $\mathbf{h}^{(0)} = \mathbf{x} \in \mathbb{R}^{D_x}$. The
observed output $\mathbf{y} \in \mathbb{R}^{D_y}$ is modeled as a
noisy observation of the deepest latent representation,
$\mathbf{y} = \mathbf{h}^{(L)} + \boldsymbol{\epsilon}_{y}$ with
$\boldsymbol{\epsilon}_{y} \sim \mathcal{N}(\mathbf{0}, \sigma_y^2
\mathbf{I})$. The latent representations are defined recursively by
$\mathbf{h}^{(\ell)} =
f^{(\ell)}(\mathbf{h}^{(\ell-1)}) + \boldsymbol{\epsilon}^{(\ell)}$
for $\ell=1,\dots,L$, where
$\boldsymbol{\epsilon}^{(\ell)} \sim \mathcal{N}(\mathbf{0},
\sigma_\ell^2 \mathbf{I})$ is the additive Gaussian noise at layer
$\ell$, and $f^{(\ell)} : \mathbb{R}^{D_{\ell-1}} \to
\mathbb{R}^{D_\ell}$ is a vector-valued GP mapping. Although this
hierarchy increases expressive power, training becomes significantly
harder as depth increases \citep{duvenaud2014deep, dunlop2018howdeep}.

\vspace{-0.05in}
\section{From Discrete-Depth Pathology to the Residual
Bridge}\label{from-discrete-depth-pathology-to-the-residual-bridge}
\vspace{-0.05in}
\label{sec:analysis}

\paragraph{Depth and Lipschitz Intuition.}\label{subsec:lipschitz_intuition}

Repeated composition can amplify parameter perturbations. To characterize
this possibility, we analyze the Lipschitz
constant of the standard DGP evidence lower bound (ELBO) under variational inference \citep{salimbeni2017doubly}. Since this objective is typically optimized using Monte Carlo estimates, we work with a single-sample estimator:
\vspace{-0.1in}
\begin{align}
\hat{\mathcal{L}}_{\mathrm{DGP}}(\bm{\theta})
    = \log p(\y \mid \f^{(L)})
    - \sum_{\ell=1}^{L} \mathrm{KL}\!\bigl[\cdot \,\Vert\, \cdot\bigr].
\end{align}
Here, the first term is the data-fit term and the second is a
Kullback--Leibler (KL) regularizer. We suppress the KL arguments because
their exact form depends on the chosen approximation. Moreover,
$\bm{\theta}$ denotes all variational parameters and kernel
hyperparameters, and $\f^{(L)}$\footnote{Throughout
  the paper, we use the same notation for Monte Carlo samples and their
  underlying random variables or function values whenever the
  distinction is clear from context.} is sampled from the variational
distribution over the final-layer latent function values, given by
\begin{align*}
\f^{(L)} = f^{(L)}\!\bigl(f^{(L-1)}(\cdots f^{(1)}(\x)\cdots )\bigr).
\end{align*}

Let $L_{\mathrm{DGP}}$ denote the Lipschitz constant of
$\hat{\mathcal L}_{\mathrm{DGP}}$ on the parameter domain under consideration,
with Monte Carlo randomness held fixed. The following proposition gives
an assumption-dependent upper bound on this sensitivity.

\begin{proposition}\label{prop:lipschitz_explosion}
Consider an $L$-layer DGP whose sampled layer map at layer $\ell$ is
$\lambda_\ell$-Lipschitz with respect to its input. Assume in
addition that the output-gradient term and the layerwise
parameter-Jacobian term are uniformly bounded, and that the KL
regularizer decomposes layerwise with uniformly bounded block
gradients, as stated precisely in
Appendix~\ref{app:optimization_stability}. Then the Lipschitz constant
of the single-sample DGP objective satisfies
\vspace{-0.1in}
\begin{align*}
L_{\mathrm{DGP}}
\le
C\left(\sum_{k=1}^{L}\prod_{\ell=k+1}^{L}\lambda_\ell^2
\right)^{1/2}
+
O(\sqrt{L}),
\end{align*}
where $C$ is independent of $L$, and the $O(\sqrt{L})$ term
captures the contribution of the KL regularizer under that layerwise
decomposition. The proof is given in
Appendix~\ref{app:optimization_stability}.
\end{proposition}

Proposition~\ref{prop:lipschitz_explosion} shows that the depth
dependence of $L_{\mathrm{DGP}}$ is driven mainly by the data-fit
term, rather than the KL regularization, and may grow exponentially with
depth. Even at the level of a single-sample pathwise estimator, this
implies increasing sensitivity to parameter perturbations as depth
grows, making standard DGP optimization progressively more
ill-conditioned, with estimator variance potentially worsening the
effect.
Repeated discrete composition allows small perturbations to be
multiplied from one layer to the next. That is the pathology we want to
repair before introducing a continuous-depth model.

\paragraph{Residual DGP as a First Repair.}\label{subsec:continuous_limit}

This motivates a residual formulation of DGPs
\citep{wyrwalresidual}, in which each layer updates the current latent
representation additively rather than fully remapping it:
\begin{align*}
\mathbf{h}^{(\ell)} = \mathbf{h}^{(\ell-1)}
+ f^{(\ell)}(\mathbf{h}^{(\ell-1)})
+ \boldsymbol{\epsilon}^{(\ell)},
\qquad \ell = 1,\dots,L.
\end{align*}
Because the identity path is retained, each residual function only needs
to model a correction, reducing perturbation amplification at each
layer. However, the model remains a finite discrete composition, so its
sensitivity can still be influenced by repeated Jacobian products and
may accumulate with depth.
Residual DGP is therefore the right first repair, but it still treats
depth as a finite stack of separately parameterized residual blocks.
This motivates the continuous-depth redesign introduced in
Section~\ref{sec:models}.

\section{Continuous-Depth Modeling and
Inference}\label{continuous-depth-modeling-and-inference}

\vspace{-0.05in}
\label{sec:models}
\vspace{-0.05in}

If depth is meant to express gradual representation evolution, then a
finite collection of separately parameterized residual layers is only
one way to represent that evolution. We therefore replace the discrete
layer index by a depth variable $t \in [0,T]$, so the latent
representation evolves gradually rather than through a fixed stack of
residual blocks.

The resulting model is a redesign motivated by the residual view, not a
claimed scaling limit of a discrete DGP. Its advantage is that it
separates the latent dynamics from the choice of parameterization
across depth. Once this shared continuous-depth system is fixed, the
main modeling question becomes what prior should be placed on its
depth-indexed drift family.

Concretely, the latent
representation evolves over depth $t \in [0,T]$ according to the
It\^o SDE
\begin{equation}
    d\mathbf{h}_t
    = f_t(\mathbf{h}_t)\,dt
    + \sigma_t\, d\mathbf{w}_t,
    \label{eq:gp_sde}
\end{equation}
where $\mathbf{w}_t$ is a Wiener process and $\sigma_t\,d\mathbf{w}_t$
is the continuous counterpart of the layerwise noise
$\boldsymbol{\epsilon}^{(\ell)}$.\footnote{If discretized by Euler--Maruyama
with unit step size, Eq.~\eqref{eq:gp_sde} has the update form of a residual
DGP layer, but uses a shared depth-indexed drift prior rather than
independent layer GPs.} The central modeling object is
therefore the depth-indexed drift family
$\{f_t\}_{t \in [0,T]}$, with each $f_t$ acting as a vector field on
the current latent state. This formulation cleanly separates the latent
dynamics from the way dependence across depth is parameterized.

The paper studies two GP-based priors on this same drift family. CDGP
places a direct GP prior on the joint state-depth drift, while FlowDGP
generates the family by placing a GP prior on a base drift and evolving
it across depth. The common latent SDE is the same in both cases; only
the prior on $\{f_t\}$ changes.

\vspace{-0.05in}
\subsection{CDGP: Separable Kernel
Instantiation}\label{cdgp-separable-kernel-instantiation}
\label{subsec:cdgp}
\vspace{-0.05in}

The continuous-depth model in Section~\ref{sec:models} leaves open the
prior on the drift family $\{f_t\}$. The most direct answer is to
write that family as a single function of state and depth,
$f_t(\mathbf{h}) \equiv f(\mathbf{h},t)$, and place a GP prior on
this joint state--depth drift with a separable kernel:
\begin{align*}
k\big((\mathbf{h},t),(\mathbf{h}',t')\big)
= k_h(\mathbf{h},\mathbf{h}')\,k_t(t,t').
\end{align*}
We term this first prior the continuous-depth DGP (CDGP). CDGP is the
most direct implementation of the continuous-depth idea: place a GP
prior on the drift family itself and infer that object with standard
sparse-GP machinery. The separable structure enables a Kronecker
factorization of the inducing-point covariance, making inference
scalable. However, it restricts how the drift can vary jointly with
state and depth: the covariance between drift values at different
depths factorizes into a product of state-only and depth-only terms,
preventing the kernel from capturing interactions in which the spatial
structure of the drift changes with depth. This motivates a second,
more structured prior on the same family $\{f_t\}$.

\vspace{-0.05in}
\subsection{FlowDGP: CNF-Transformed
GP}\label{flowdgp-cnf-transformed-gp}
\label{subsec:flowdgp}
\vspace{-0.05in}

CDGP models the shared drift family directly, but its separable kernel
is deliberately restrictive. We therefore consider a second prior on
the same family $\{f_t\}$, one that enriches the depth dependence
without incurring the cost of a non-separable joint kernel. We draw a base drift
$f_0$ from a GP prior and evolve it across depth via a continuous
normalizing flow (CNF). A shared vector field $g_{\bm{\phi}}$ defines
an ODE whose solution at depth $t$ gives the evolved drift $f_t$.
Here $g_{\bm{\phi}}$ acts on a finite-dimensional representation of
the drift itself, not on the latent state $\h_t$; in practice it is
applied pointwise to sampled drift values such as $f_0(\x_n)$, and we
parameterize it as a small multilayer perceptron (MLP) with parameters
$\bm{\phi}$. For regression with latent state dimension $D_y$, each
such pointwise drift value lies in
$\mathbb R^{D_y}$, so $g_{\bm{\phi}}:\mathbb R^{D_y}\to\mathbb
R^{D_y}$. Concretely,
\begin{equation}
  \frac{\mathrm{d} f_t}{\mathrm{d} t}
    = g_{\bm{\phi}}(f_t), \qquad
  f_0 \sim \mathcal{GP}(0, k_0).
  \label{eq:flow_ode}
\end{equation}
Here $f_0$ is learned through its inducing variables rather than fixed
as a function.
The displayed equation is a function-level idealization; in inference we
apply the same flow pointwise to sampled drift values along each
propagated data trajectory.

FlowDGP should therefore be understood as a second prior on the same
latent SDE in Eq.~\eqref{eq:gp_sde}, not as a separate model family.
The contrast is simple: CDGP parameterizes depth dependence directly in
the kernel, whereas FlowDGP parameterizes it by evolving a
base drift through the neural flow $g_{\bm{\phi}}$. The evolution is
deterministic given $f_0$, keeping inference on a single base GP.
This avoids both the large joint kernel matrix of a non-separable kernel
and the additional path inference required by a diffusion over $f_t$.
The choice is motivated by tractability, rather than optimality among
flow or diffusion architectures. Depth dependence is
instead captured by the ODE integration of $g_{\bm{\phi}}$, which
adds solver cost but, as we show in Section~\ref{subsec:benchmark},
yields substantially faster training than CDGP. This construction is also
related in spirit to differential GP flows \citep{hegde2019deep}, but
here it is used to define a prior over the depth-indexed drift that sits
inside the same continuous-depth regression framework as CDGP.

\vspace{-0.05in}
\subsection{Variational Inference}
\label{subsec:vi_new}
\vspace{-0.05in}

CDGP and FlowDGP are two priors on the same continuous-depth latent
system. Their shared inference structure is most transparent at the
level of the common simulator rather than model-specific algebra.
For each input $\x_n$, both models produce a terminal latent state
$\h_T^{(n)}$ and then apply the same readout likelihood:
\begin{equation}
\y_n \sim p(\y_n \mid \h_T^{(n)}),
\qquad
\h_T^{(n)} = \Phi(\x_n, \vf, \xi_n; \bm{\theta}),
\end{equation}
where $\vf$ denotes the latent GP object, $\xi_n$ collects the Gaussian exogenous noise used during path simulation, and $\bm{\theta}$ collects the remaining model, kernel, and variational parameters. Here $\Phi(\x_n, \vf, \xi_n; \bm{\theta})$ denotes the deterministic map that propagates the latent state from depth $0$ to depth $T$ once $\vf$ and $\xi_n$ are given.

The two models differ only in what object carries the GP prior. In CDGP, the latent GP object is the joint state--depth drift $\vf = f(\h,t)$, so inference is performed directly on the drift defined over the joint state--depth domain. In FlowDGP, the latent GP object is
the base drift $\vf = f_0$; the full depth-indexed family $(f_t)$ is
then obtained deterministically by evolving $f_0$ through the flow ODE
in Eq.~\eqref{eq:flow_ode}. This is why FlowDGP places inducing points
only on the base GP, whereas CDGP places them on the joint
state--depth input.

Both models use the same variational recipe. Introduce inducing variables $\vu$ for the latent GP object $\vf$, place a Gaussian variational distribution
\begin{equation}
q(\vu) = \mathcal{N}(\mathbf{m}, \mathbf{S}),
\end{equation}
and keep the GP conditional exact, so that the full approximation takes the form
\begin{equation}
q(\vf, \vu, \xi) = q(\vu)\,p(\vf \mid \vu)\,p(\xi).
\end{equation}
All model-specific structure remains inside the simulator $\Phi$: once $\vf$ is sampled from its sparse variational posterior and $\xi_n$ is sampled from its Gaussian law, the latent path and the
terminal state are obtained by deterministic propagation.
The evidence lower bound is
\begin{equation}
\mathcal{L}
=
\sum_{n=1}^{N}
\mathbb{E}_{q(\vu)\,p(\vf \mid \vu)\,p(\xi_n)}
\Big[
\log p\bigl(\y_n \mid \Phi(\x_n, \vf, \xi_n; \bm{\theta})\bigr)
\Big]
-
\mathrm{KL}\!\bigl[q(\vu)\,\Vert\,p(\vu)\bigr].
\label{eq:generic_gp_sim_elbo}
\end{equation}
This is the common inference structure: inducing-point variational
inference for GP-driven continuous-depth simulators.
In practice, $\Phi$ is implemented on the discretized depth grid
$0=\tau_0<\tau_1<\cdots<\tau_J=T$ using a left-point
Euler--Maruyama scheme. Writing
$\Delta\tau_j=\tau_j-\tau_{j-1}$ and
$\Delta\w_j=\w_{\tau_j}-\w_{\tau_{j-1}}$, one simulated path is
generated by first sampling the relevant GP object and then updating
\begin{equation}
\h_{\tau_j}^{(n)}
=
\h_{\tau_{j-1}}^{(n)}
+
f_{\tau_{j-1}}^{(n)}\!\bigl(\h_{\tau_{j-1}}^{(n)}\bigr)\,\Delta\tau_j
+
\sigma_{\tau_j}\Delta\w_j,
\qquad j=1,\ldots,J,
\label{eq:em_generic}
\end{equation}
starting from $\h_0^{(n)}=\x_n$. In CDGP, the drift value
$f_{\tau_{j-1}}^{(n)}(\h_{\tau_{j-1}}^{(n)})$ is read directly from
the sampled joint GP $f(\h,t)$; in FlowDGP, one first samples the
base GP $f_0$, transports it through the flow ODE to obtain
$(f_t)$, and then uses the transported drift in the same update.

With $S$ Monte Carlo samples, the optimized discretized objective is
\begin{equation}
\widehat{\mathcal{L}}
=
\sum_{n=1}^N
\frac{1}{S}\sum_{s=1}^S
\log p\!\bigl(\y_n \mid \h_T^{(n,s)}\bigr)
-
\mathrm{KL}\!\bigl[q(\vu)\,\Vert\,p(\vu)\bigr],
\label{eq:generic_gp_sim_mc}
\end{equation}
where each $\h_T^{(n,s)}$ is obtained by the corresponding simulated
path under Eq.~\eqref{eq:em_generic}.

For CDGP, we instantiate this construction with
$\vu = f(\widetilde{\mathbf{Z}})$ on the Cartesian product
$\widetilde{\mathbf{Z}} = \mathbf{Z}_h \times \mathbf{Z}_t$, where the
separable kernel gives the Kronecker prior
$\mathbf{K}_{uu} = \mathbf{K}_h(\mathbf{Z}_h,\mathbf{Z}_h)\otimes
\mathbf{K}_t(\mathbf{Z}_t,\mathbf{Z}_t)$. The simulator $\Phi$
propagates the latent state by discretizing the SDE in
Eq.~\eqref{eq:gp_sde} and evaluating the sampled drift
$f(\h_{\tau_{j-1}}^{(n)}, \tau_{j-1})$ along the path.
For FlowDGP, we instantiate the same construction with
$\u = f_0(\mathbf{Z})$ for the base GP alone. After sampling the base
drift from its sparse variational posterior, $\Phi$ first transports
it across depth through the ODE in Eq.~\eqref{eq:flow_ode} to obtain the
family $(f_t)$, and then propagates the latent state with those
transported drift values. CDGP and FlowDGP therefore share the same
likelihood, the same simulator-level variational logic, and the same
Monte Carlo training principle; they differ only in the latent GP object
being approximated and in the domain on which inducing variables are
placed.

\vspace{-0.05in}
\subsection{Optimization Stability}
\label{subsec:stability}
\vspace{-0.05in}

We now ask what kind of sensitivity statement survives once the deep
stack is replaced by a latent SDE. The relevant quantity is the pathwise
sensitivity of the SDE solution to parameter perturbations. This
analysis clarifies what replaces the repeated Jacobian products of a
discrete stack, but it is not itself a training guarantee for CDGP or
FlowDGP.
\begin{theorem}\label{thm:pathwise_stability}
Consider the SDE $ d\mathbf{h}_t = f(\mathbf{h}_t, t; \bm{\theta})\,dt + \sigma_t\,d\mathbf{w}_t.$ Let $\mathbf{h}_t^{(1)}$ and $\mathbf{h}_t^{(2)}$
denote the strong solutions associated with parameters $\bm{\theta}_1$
and $\bm{\theta}_2$, respectively, driven by the same Brownian motion.
Assume that the diffusion coefficient $\sigma_t$ is shared across the
two systems, and that $f$ is Lipschitz continuous in the state
variable with constant $L_h$, and in the parameter variable with
constant $L_\theta$. Then, for any $t \in [0,T]$, \begin{equation}
    \big\|\mathbf{h}_t^{(1)} - \mathbf{h}_t^{(2)}\big\|
    \le
    L_\theta\,\phi(L_h,t)
    \big\|\bm{\theta}_1 - \bm{\theta}_2\big\|,
\end{equation} where $\phi(L_h,t) = (e^{L_h t}-1)/L_h$ for $L_h>0$,
with continuous extension $\phi(0,t)=t$, for almost every sample path.
See Appendix~\ref{app:proof_thm_3.1} for the proof.
\end{theorem}

Theorem~\ref{thm:pathwise_stability} is a generic pathwise sensitivity
statement. Unlike the discrete-depth proposition, it does not grow
through repeated products of layerwise Jacobians. It still carries an
exponential dependence on $T$, however, so a depth-uniform conclusion
requires extra structure.

\begin{corollary}\label{cor:contractive_stability}
Under the assumptions of
Theorem~\ref{thm:pathwise_stability}, suppose in
addition that the drift satisfies a one-sided contractivity condition on
$[0,T]$, i.e., there exists $\mu>0$ such that, for all
$\mathbf{h}_1,\mathbf{h}_2$ and $t\in[0,T]$, \begin{equation}
\bigl\langle
f(\mathbf{h}_1,t;\bm{\theta})-f(\mathbf{h}_2,t;\bm{\theta}),
\mathbf{h}_1-\mathbf{h}_2
\bigr\rangle
\le
-\mu \|\mathbf{h}_1-\mathbf{h}_2\|^2 .
\end{equation} Then, \begin{equation}
  \|\mathbf{h}_T^{(1)}-\mathbf{h}_T^{(2)}\|
  \le
  \frac{L_\theta}{\mu}\sqrt{1-e^{-\mu T}}\,
  \|\bm{\theta}_1-\bm{\theta}_2\|
  \le
  \frac{L_\theta}{\mu}\|\bm{\theta}_1-\bm{\theta}_2\|,
\end{equation} for almost every sample path. See
Appendix~\ref{app:proof_cor} for the proof.
\end{corollary}
Under the contractive regime, the pathwise sensitivity is bounded by
$L_\theta/\mu$. If the loss gradient is bounded by $C$, the data-fit
term at terminal time satisfies, denoting by $L_{\mathrm{SDE}}$ the
Lipschitz constant of the terminal data-fit term with respect to
$\bm{\theta}$,
\begin{align}
L_{\mathrm{SDE}} \leq \frac{C\,L_\theta}{\mu}.
\end{align}
The qualitative difference from discrete DGPs is the object that
controls sensitivity. Proposition~\ref{prop:lipschitz_explosion} is
driven by repeated discrete composition, whereas the continuous-depth
formulation yields bounds in terms of a single drift field. The
contractive corollary is intentionally stronger and more specialized:
neither CDGP nor FlowDGP explicitly enforces that regime. We therefore
use these results to interpret the redesign, while
Section~\ref{subsec:toy_example} shows empirically that the proposed
continuous-depth models remain stable across the depth ranges we study.
Appendix~\ref{app:contraction-diagnostic} provides a controlled SDE
illustration of Corollary~\ref{cor:contractive_stability}.

\vspace{-0.05in}
\section{Experiments}
\vspace{-0.05in}

The experiments are organized as a sequence of evidence for the paper's
main claim. First, the toy study makes the discrete-depth pathology
visible and tests whether residualization is enough to repair it.
Second, the benchmark study asks whether the continuous-depth redesign
remains competitive on realistic regression tasks once residual DGP is
granted as the strongest discrete-depth baseline. Third, the large-scale
study asks which part of the continuous-depth family remains practical
when direct GP baselines become too expensive to train.

\vspace{-0.05in}
\subsection{Toy Example}
\vspace{-0.05in}

\label{subsec:toy_example}

\begin{figure}[t!]
    \centering
    \includegraphics[width=0.9\linewidth]{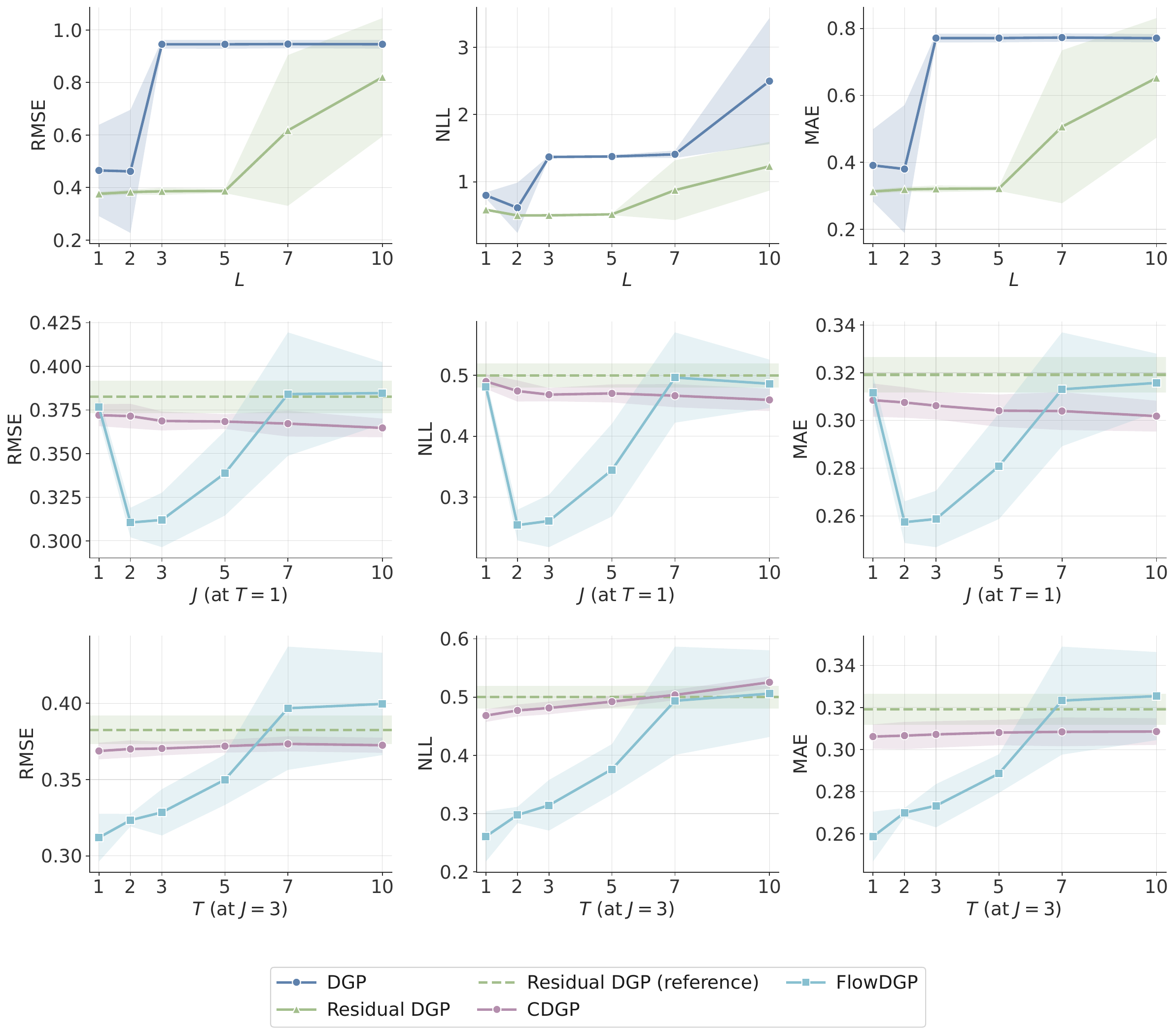}
    \vspace{-0.15in}
    \caption{\textbf{Toy regression results.} Top: test root mean squared error (RMSE), negative log-likelihood (NLL), and mean absolute error (MAE) versus depth $L$ for standard DGP and residual DGP. Middle: CDGP and FlowDGP versus $J$ at fixed $T=1$, with the dashed line showing the two-layer residual DGP reference. Bottom: CDGP and FlowDGP versus $T$ at fixed $J=3$. Shaded regions denote one standard deviation across random seeds.}
    \vspace{-0.2in}
    \label{fig:toy_example}
\end{figure}

We begin with a synthetic regression task designed to expose the depth
pathology directly. The point of this experiment is diagnostic: make the
failure of plain DGP visible, show that residualization is the right
first repair, and then ask whether the continuous-depth models remain
stable beyond it. The dataset comprises $800$
input--output pairs generated from a nonlinear function with additive Gaussian noise, partitioned into $300$ training and $500$ test samples. We report test negative log-likelihood (NLL), root mean squared error (RMSE), and mean absolute error (MAE), averaged over random seeds. Details of the dataset and hyperparameter settings are provided in Appendix~\ref{app:dataset} and Appendix~\ref{app:setting}, respectively.

Figure~\ref{fig:toy_example} should be read row by row. The top row
tests the failure mode of increasing discrete depth, the middle row asks
whether the continuous-depth models stay stable as the numerical
resolution $J$ changes, and the bottom row asks whether they remain
well behaved as the terminal depth $T$ increases.

The top row of Figure~\ref{fig:toy_example} shows DGP
and Residual DGP as the number of layers $L$ increases from $1$ to
$10$. Standard DGP deteriorates sharply across all three metrics at larger $L$,
consistent with the exponential Lipschitz growth identified in
Proposition~\ref{prop:lipschitz_explosion}. Residual DGP
is more robust due to its identity shortcut but still degrades at larger
depths, confirming that residualization helps materially but does not
remove the limitations of discrete composition.

The middle row fixes $T=1$ and varies the number of discretization
steps $J$ from $1$ to $10$. CDGP remains stable
across all metrics, matching or outperforming the two-layer
Residual DGP configuration (dashed line). FlowDGP achieves
lower NLL than CDGP at $J\le5$.

The bottom row fixes $J=3$ and varies $T$ from $1$ to $10$. CDGP
maintains stable performance as $T$ increases, and FlowDGP shows
better RMSE and MAE at $T$ values.
This behavior is qualitatively consistent with
the sensitivity perspective of
Section~\ref{subsec:stability}, but the empirical point stands on
its own: once depth is parameterized continuously, the models remain
well behaved across the range we test.

\vspace{-0.05in}
\subsection{Benchmark Datasets}
\label{subsec:benchmark}
\vspace{-0.05in}

\begin{table}[t]
\centering
\caption{Regression results on benchmark datasets. We report mean $\pm$ std over multiple random seeds. Bold indicates the best result; underline indicates the second best. Training time is in seconds. Rows highlighted in light blue correspond to our methods. Method abbreviations: DKL = deep kernel learning, TDGP = thin DGP, TGP = transformed GP, NODE = neural ODE, and NSDE = neural SDE.}
\label{tab:results}
\setlength{\tabcolsep}{5pt}
\renewcommand{\arraystretch}{1.12}
\resizebox{\textwidth}{!}{%
\begin{tabular}{llcccccc}
\toprule
& & \multicolumn{6}{c}{\textbf{Dataset}} \\
\cmidrule(lr){3-8}
\textbf{Metric} & \textbf{Method}
  & \texttt{heavy\_tailed}
  & \texttt{heteroscedastic}
  & \texttt{nonstationary}
  & \texttt{sarcos\_small}
  & \texttt{uci\_bike\_small}
  & \texttt{uci\_energy\_small} \\
\midrule
\rowcolor{ourblue}
\multirow{11}{*}{\rotatebox[origin=c]{90}{\textbf{RMSE $\downarrow$}}}
& FlowDGP            & $\underline{0.235 \pm 0.003}$ & $\mathbf{0.819 \pm 0.011}$    & $0.377 \pm 0.123$             & $\underline{0.323 \pm 0.028}$    & $\mathbf{0.934 \pm 0.121}$            & $0.804 \pm 0.009$ \\
\rowcolor{ourblue}
& CDGP               & $\mathbf{0.231 \pm 0.005}$    & $0.856 \pm 0.077$             & $0.634 \pm 0.008$             & $0.332 \pm 0.004$                & $1.122 \pm 0.012$                     & $\underline{0.802 \pm 0.001}$ \\
& Heteroskedastic GP & $0.306 \pm 0.002$             & $\underline{0.823 \pm 0.006}$ & $0.654 \pm 0.008$             & $0.345 \pm 0.010$                & $1.152 \pm 0.006$                     & $\mathbf{0.797 \pm 0.001}$ \\
& Residual DGP       & $0.311 \pm 0.005$             & $\underline{0.823 \pm 0.003}$ & $0.676 \pm 0.056$             & $0.347 \pm 0.006$                & $1.126 \pm 0.005$                     & $\underline{0.802 \pm 0.001}$ \\
& Sparse GP          & $0.309 \pm 0.001$             & $1.023 \pm 0.008$             & $0.756 \pm 0.006$             & $1.051 \pm 0.003$                & $1.162 \pm 0.004$                     & $\underline{0.802 \pm 0.002}$ \\
& TGP & $0.305 \pm 0.003$ & $0.893 \pm 0.006$ & $0.464 \pm 0.008$ & $0.383 \pm 0.028$ & $0.986 \pm 0.026$ & $0.827 \pm 0.004$ \\
& DGP                & $0.321 \pm 0.002$             & $1.012 \pm 0.002$             & $0.745 \pm 0.002$             & $1.050 \pm 0.001$                & $1.182 \pm 0.031$                     & $\underline{0.802 \pm 0.027}$ \\
& NODE               & $0.544 \pm 0.197$             & $0.897 \pm 0.036$             & $0.356 \pm 0.029$ & $0.351 \pm 0.010$                & $0.993 \pm 0.075$                     & $0.964 \pm 0.099$ \\
& TDGP               & $0.307 \pm 0.008$             & $0.894 \pm 0.078$             & $\underline{0.271 \pm 0.091}$ & $0.393 \pm 0.014$                & $1.037 \pm 0.034$                     & $\underline{0.802 \pm 0.001}$ \\
& DKL                & $0.437 \pm 0.029$             & $0.928 \pm 0.018$             & $0.548 \pm 0.065$             & $0.543 \pm 0.162$                & $\underline{0.956 \pm 0.175}$         & $1.593 \pm 0.161$ \\
& NSDE               & $0.480 \pm 0.043$             & $0.962 \pm 0.037$             & $\mathbf{0.211 \pm 0.010}$    & $\mathbf{0.302 \pm 0.007}$       & $1.052 \pm 0.033$                     & $1.317 \pm 0.146$ \\
\midrule
\rowcolor{ourblue}
\multirow{11}{*}{\rotatebox[origin=c]{90}{\textbf{NLL $\downarrow$}}}
& FlowDGP            & $\underline{0.329 \pm 0.022}$    & $\mathbf{1.248 \pm 0.008}$    & $\mathbf{0.411 \pm 0.286}$    & $1.352 \pm 0.692$                & $1.868 \pm 0.203$                     & $1.229 \pm 0.021$ \\
\rowcolor{ourblue}
& CDGP               & $0.355 \pm 0.037$ & $1.285 \pm 0.076$             & $0.979 \pm 0.006$ & $\mathbf{0.376 \pm 0.011}$       & $1.585 \pm 0.016$                     & $1.223 \pm 0.003$ \\
& Heteroskedastic GP & $0.490 \pm 0.015$             & $1.287 \pm 0.005$             & $1.013 \pm 0.012$             & $\underline{0.440 \pm 0.025}$    & $\underline{1.569 \pm 0.006}$            & $1.258 \pm 0.011$ \\
& Residual DGP       & $0.490 \pm 0.022$             & $1.254 \pm 0.004$             & $1.032 \pm 0.076$             & $\underline{0.440 \pm 0.017}$    & $1.605 \pm 0.010$                     & $1.232 \pm 0.004$ \\
& Sparse GP          & $0.501 \pm 0.005$             & $\underline{1.253 \pm 0.006}$ & $0.987 \pm 0.008$             & $1.470 \pm 0.001$                & $1.605 \pm 0.002$                     & $\underline{1.211 \pm 0.001}$ \\
& TGP & $\mathbf{0.258 \pm 0.007}$ & $1.326 \pm 0.008$ & $\underline{0.897 \pm 0.005}$ & $0.707 \pm 0.018$ & $\mathbf{1.519 \pm 0.053}$ & $1.257 \pm 0.007$ \\
& DGP                & $0.536 \pm 0.036$             & $1.442 \pm 0.002$             & $1.125 \pm 0.002$             & $1.470 \pm 0.001$                & $1.800 \pm 0.022$                     & $1.218 \pm 0.003$ \\
& NODE               & $0.870 \pm 0.427$             & $42.231 \pm 9.592$            & $30.059 \pm 3.559$            & $26.245 \pm 3.366$               & $52.447 \pm 23.808$                   & $21.634 \pm 11.721$ \\
& TDGP               & $0.357 \pm 0.051$             & $5.317 \pm 1.026$             & $1.080 \pm 1.323$             & $1.278 \pm 0.103$                & $5.585 \pm 0.573$                     & $\mathbf{1.204 \pm 0.002}$ \\
& DKL                & $25.536 \pm 2.319$            & $34.401 \pm 3.518$            & $6.608 \pm 0.271$             & $6.605 \pm 0.361$                & $4.604 \pm 1.652$                     & $9.590 \pm 1.391$ \\
& NSDE               & $0.828 \pm 0.259$             & $3.527 \pm 1.601$             & $2.058 \pm 0.462$             & $9.516 \pm 0.390$                & $1.584 \pm 0.049$         & $\mathbf{1.204 \pm 0.003}$ \\
\midrule
\rowcolor{ourblue}
\multirow{11}{*}{\rotatebox[origin=c]{90}{\textbf{Train Time}}}
& FlowDGP            & $27.490 \pm 0.940$            & $19.130 \pm 2.240$            & $27.690 \pm 10.150$           & $14.870 \pm 1.694$               & $12.887 \pm 0.231$                    & $31.910 \pm 7.909$ \\
\rowcolor{ourblue}
& CDGP               & $201.330 \pm 1.130$           & $189.180 \pm 4.420$           & $126.930 \pm 1.670$           & $202.643 \pm 12.019$             & $93.149 \pm 9.607$                    & $178.598 \pm 23.021$ \\
& Heteroskedastic GP & $21.870 \pm 0.600$            & $26.060 \pm 1.910$            & $22.600 \pm 0.730$            & $23.375 \pm 5.970$               & $26.174 \pm 8.227$                    & $31.023 \pm 7.772$ \\
& Residual DGP       & $47.880 \pm 4.790$            & $74.580 \pm 8.340$            & $67.960 \pm 2.340$            & $48.065 \pm 2.814$               & $51.406 \pm 6.523$                    & $117.847 \pm 16.674$ \\
& Sparse GP          & $11.810 \pm 0.160$            & $11.300 \pm 1.280$            & $11.510 \pm 0.270$            & $10.052 \pm 1.869$               & $9.191 \pm 0.155$                     & $9.236 \pm 0.264$ \\
& TGP & $129.878 \pm 8.237$ & $129.303 \pm 7.154$ & $137.329 \pm 5.959$ & $134.859 \pm 8.759$ & $140.042 \pm 9.254$ & $130.997 \pm 9.929$ \\
& DGP                & $71.360 \pm 4.350$            & $124.420 \pm 6.030$           & $119.190 \pm 2.410$           & $70.540 \pm 5.532$               & $71.327 \pm 8.291$                    & $128.722 \pm 11.922$ \\
& NODE               & $13.470 \pm 0.230$            & $10.590 \pm 0.480$            & $3.840 \pm 0.050$             & $4.822 \pm 0.283$                & $6.927 \pm 2.081$                     & $5.573 \pm 2.423$ \\
& TDGP               & $52.386 \pm 6.236$            & $37.278 \pm 15.171$           & $34.389 \pm 24.178$           & $32.167 \pm 21.912$              & $26.451 \pm 2.199 $                   & $72.789 \pm 12.371$   \\
& DKL                & $6.570 \pm 0.460$             & $6.880 \pm 1.000$             & $8.990 \pm 0.540$             & $8.854 \pm 0.456$                & $4.051 \pm 0.223$                     & $3.995 \pm 0.044$ \\
& NSDE               & $74.350 \pm 9.580$            & $67.580 \pm 9.380$            & $69.420 \pm 6.630$            & $73.253 \pm 2.473$               & $63.443 \pm 8.320$                    & $79.741 \pm 10.047$ \\
\bottomrule
\end{tabular}%
}
\vspace{-0.2in}
\end{table}

The benchmark study asks a different question from the toy experiment:
once residual DGP is granted as the strongest discrete-depth repair, are
the proposed continuous-depth models still competitive on realistic
regression tasks? We therefore evaluate all methods on six regression
datasets spanning diverse data characteristics: \texttt{heavy\_tailed},
\texttt{heteroscedastic}, and \texttt{nonstationary} are synthetic
datasets designed to test specific modeling challenges, while
\texttt{sarcos\_small}, \texttt{uci\_bike\_small}, and
\texttt{uci\_energy\_small} are subsampled versions of real-world
datasets (see Appendix~\ref{app:dataset} for details). We report test
NLL, RMSE, and training time averaged over random seeds. The
benchmark baselines are sparse GP, heteroscedastic GP
\citep{le2005heteroscedastic}, standard DGP, residual DGP
\citep{wyrwalresidual}, deep kernel learning (DKL)
\citep{wilson2016deep}, thin DGP (TDGP) \citep{de2023thin}, neural ODE
(NODE) \citep{chen2018neural}, neural SDE (NSDE)
\citep{li2020scalable}, and transformed GP (TGP)
\citep{maronas2021transforming}.

Table~\ref{tab:results} summarizes the results. Three patterns matter
most. First, the continuous-depth models are not only stable in the toy
setting; they are genuinely competitive on realistic regression tasks.
FlowDGP achieves the best NLL on \texttt{heteroscedastic} and
\texttt{nonstationary}, while CDGP achieves the best NLL on
\texttt{sarcos\_small}. TGP has the best NLL on \texttt{heavy\_tailed}
and \texttt{uci\_bike\_small}; CDGP or FlowDGP has lower RMSE than TGP
on all six datasets. Second, the comparison to
discrete-depth baselines matches the motivating story of the paper:
standard DGP performs poorly across most datasets, especially in NLL,
and residual DGP improves over DGP but still trails CDGP and FlowDGP.
Third, the two continuous-depth variants occupy different practical
regimes. FlowDGP is substantially faster than CDGP, with roughly
$5$--$14\times$ lower training time across datasets, while CDGP
remains the strongest direct GP instantiation of the framework.

The benchmark comparison also includes NODE and NSDE, which share
continuous-depth dynamics but parameterize the vector field with a neural
network rather than a GP-based prior. They play different roles. NODE is the clean
deterministic baseline for continuous depth without a GP-based prior,
but its NLL is often dramatically worse, especially on
\texttt{heteroscedastic}, \texttt{nonstationary}, and
\texttt{sarcos\_small}. NSDE is the closer stochastic rival: it achieves
very strong RMSE on \texttt{nonstationary} and \texttt{sarcos\_small},
but its NLL remains markedly weaker. The benchmark results show more than the usefulness of
continuous dynamics alone. In this small- to medium-scale regime, the
GP-based prior over the drift family materially improves probabilistic
performance relative to the neural continuous-depth
alternatives. Together with the toy study, this closes the main
comparison against both residual DGP and the non-GP continuous-depth
baselines. The next subsection isolates the larger-scale regime in
which FlowDGP, rather than CDGP, becomes the practical representative
of the continuous-depth approach.

\vspace{-0.05in}
\subsection{Large-Scale Datasets}
\vspace{-0.05in}
\label{subsec:large-scale}

\begin{figure}[t!]
    \centering
    \includegraphics[width=0.9\linewidth]{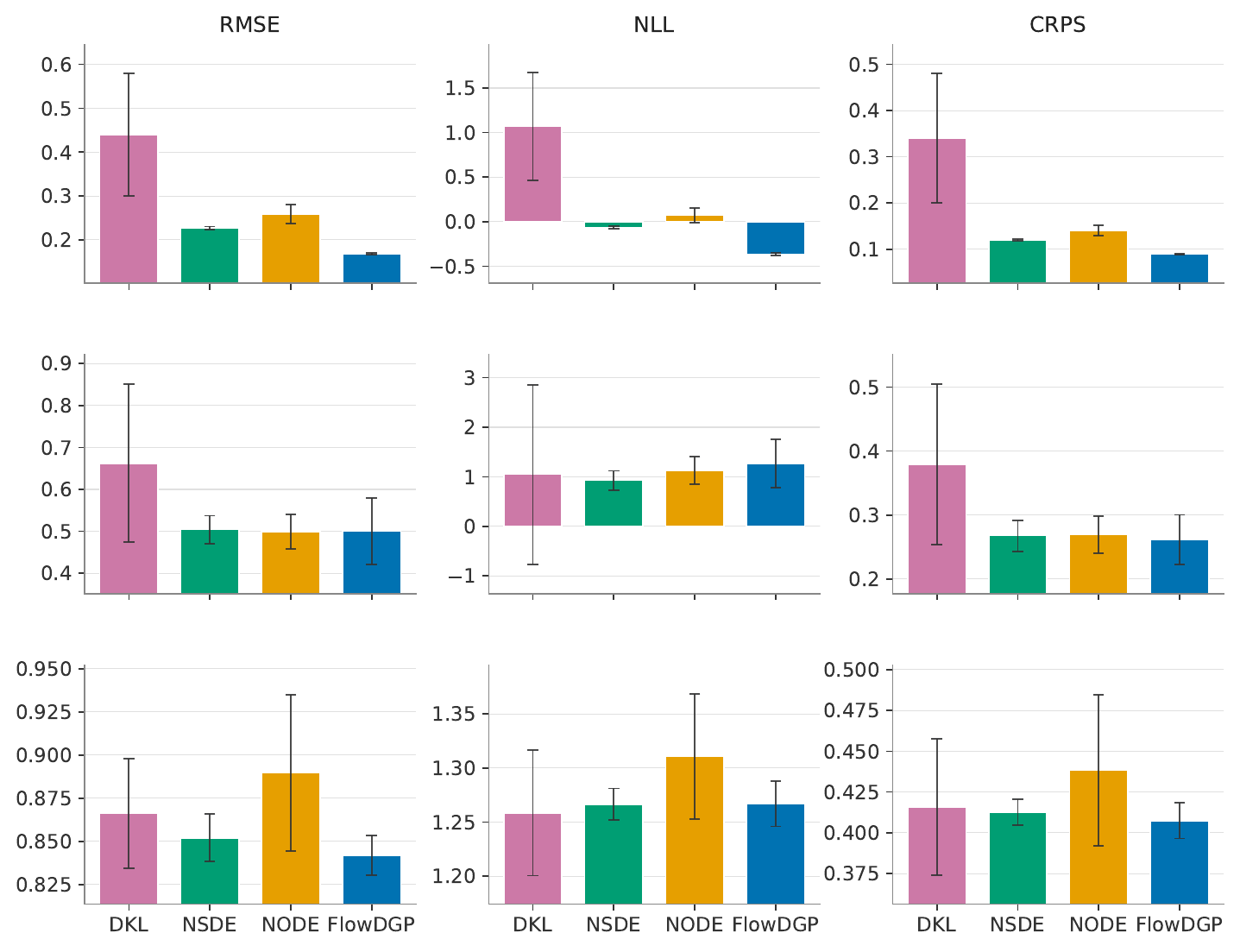}
    \vspace{-0.15in}
    \caption{Results on full-size \texttt{uci\_bike} (top), \texttt{uci\_energy} (middle), and \texttt{sarcos} (bottom) datasets comparing FlowDGP, deep kernel learning (DKL), neural ordinary differential equations (NODE), and neural stochastic differential equations (NSDE). We report root mean squared error (RMSE), negative log-likelihood (NLL), and continuous ranked probability score (CRPS) averaged over random seeds; error bars denote one standard deviation. For readability, the NLL values of DKL on \texttt{sarcos} are scaled by $0.1$ (first) and $1/3$ (second) to avoid distorting the $y$-axis scale.}
    \label{fig:large_scale}
    \vspace{-0.2in}
\end{figure}

The large-scale study isolates the practical regime of FlowDGP. We
compare against methods with neural-network components that scale
naturally to larger datasets: DKL
\citep{wilson2016deep}, neural ODE
\citep{chen2018neural}, neural SDE
\citep{li2020scalable}, and FlowDGP.
The evaluation uses the full-size versions of \texttt{uci\_bike},
\texttt{uci\_energy}, and \texttt{sarcos}
(Appendix~\ref{app:dataset}). CDGP and residual DGP are
omitted because their training cost is prohibitive at these scales; the
benchmark timings already show this trend, and
Appendix~\ref{app:complexity} explains why. We report
RMSE, NLL, and continuous ranked probability score (CRPS), averaged over
random seeds. The point is not that CDGP disappears from the
story, but that the two continuous-depth variants occupy different
cost-quality regimes: CDGP is the direct GP instantiation, whereas
FlowDGP is the scalable one. In this regime, the most meaningful
comparison is between FlowDGP and NSDE: both are stochastic
continuous-depth models, but only FlowDGP retains a GP-based prior over
the drift family.

Figure~\ref{fig:large_scale} presents the results. On
\texttt{uci\_bike}, FlowDGP achieves the best performance across RMSE,
NLL, and CRPS. On \texttt{uci\_energy}, it obtains the best CRPS and
is tied for the best RMSE, while NSDE has the best NLL. On \texttt{sarcos}, it performs
comparably to, or slightly better than, neural SDE across all three
metrics, while neural ODE and DKL provide the weakest probabilistic
performance.

Overall, the large-scale picture is sharper than a generic baseline
comparison. NODE remains a useful deterministic continuous-depth
reference, but it is not the main rival in probabilistic regression.
NSDE is the closer stochastic comparator, and the figure shows that
FlowDGP remains competitive with it in NLL and CRPS
across the datasets studied here. The practical point is
therefore not just that FlowDGP is competitive in point prediction, but
that a GP-based prior over the drift family remains useful even in the
scalable continuous-depth regime.\footnote{Additional results on
YearPredictionMSD are provided in Appendix~\ref{app:yearprediction-results}.}

\section{Related Work}
\label{sec:related}

DGPs compose GP mappings \citep{damianou2013deep}, with scalable
variational inference \citep{salimbeni2017doubly}. Residual
\citep{wyrwalresidual} and thin DGPs \citep{de2023thin} modify this
hierarchy. Deep kernel learning \citep{wilson2016deep} instead learns
neural features, while TGPs \citep{maronas2021transforming} transform a
GP through an invertible, potentially input-dependent map.
Our focus is the parameterization of latent depth.

GP networks \citep{friedman2000networks} put GP priors on conditional
relations in a discrete graphical model; nonlinear multifidelity GPs
\citep{perdikaris2017nonlinear} learn relations between observed fidelity
levels. Our depth variable is latent and continuous, without observations
required at intermediate depths.

Neural ODEs and SDEs \citep{chen2018neural,li2020scalable} parameterize
continuous-depth dynamics with neural vector fields. Deep state-space
GPs \citep{zhao2021deep} use conditional SDE representations over the
input time variable within a finite hierarchy, enabling sequential inference.
Differential GP flows \citep{hegde2019deep} use a GP posterior to determine
the drift and diffusion of an input-warping SDE, followed by a separate
output GP. In FlowDGP, a sampled base drift is evolved across latent depth,
the state diffusion is specified separately, and the terminal state
enters the likelihood directly.

\vspace{-0.1in}
\section{Conclusion}
\vspace{-0.1in}

For probabilistic regression, depth in a deep GP is better treated as
latent continuous evolution than as a long discrete stack. Plain DGPs
become increasingly awkward as depth grows, and residual DGP is the
right first repair but not the final answer.

We therefore replaced the deep stack by a latent SDE with a
depth-indexed drift family and studied two GP-based priors on that
family: CDGP as a direct separable prior, and FlowDGP as a structured
flow-evolved prior, both trained with a shared sparse variational
framework. The theory
explains why this continuous-depth formulation lives in a different
sensitivity regime from repeated discrete composition, while
the strongest depth-uniform statement remains conditional.

Empirically, the pattern is consistent across scales. Residual DGP
improves over plain DGP, but the continuous-depth models are more stable
or more competitive in the regimes we study, with FlowDGP performing
best at larger scale. Continuous depth is a useful redesign of how
depth is represented in probabilistic regression.
Relative to the natural continuous-depth neural alternatives,
especially neural ODEs and neural SDEs, the advantage is not continuous
depth alone but continuous depth with a GP-based prior over the drift
family.

\section{Limitations}
\label{sec:limitations}

On the modelling side, CDGP's separable kernel limits the joint
state--depth dependence of the drift, while FlowDGP captures this
dependence through a deterministic CNF, so uncertainty in the drift
across depths is inherited solely from the base GP $f_0$. Our empirical
evaluation also focuses on regression tasks of moderate dimensionality,
and extending the approach to classification or structured prediction is
left to future work.

\begin{ack}
Ying Li and Alexandre Thiery acknowledge support from Singapore MOE AcRF
Tier 2 grants \mbox{MOE-T2EP20123-0010} and \mbox{MOE-T2EP20223-0011}.
The contribution of Petar Djuric is supported by the National Science
Foundation (NSF) under Award 2212506.
The contribution of Michael Zhang is funded by the Research Grants Council
under the General Research Fund \#17312826 and by the University of Hong Kong
Seed Fund for PI Research -- Basic Research \#2402101367.
\end{ack}

\bibliographystyle{plainnat}
\bibliography{references}

\begin{thebibliography}{20}
\providecommand{\natexlab}[1]{#1}
\providecommand{\url}[1]{\texttt{#1}}
\expandafter\ifx\csname urlstyle\endcsname\relax
  \providecommand{\doi}[1]{doi: #1}\else
  \providecommand{\doi}{doi: \begingroup \urlstyle{rm}\Url}\fi

\bibitem[Bertin-Mahieux(2011)]{yearpredictionmsd}
Thierry Bertin-Mahieux.
\newblock {Year Prediction MSD}.
\newblock UCI Machine Learning Repository, 2011.
\newblock {DOI}: 10.24432/C50K61.

\bibitem[Candanedo(2017)]{appliances_energy_prediction_374}
Luis Candanedo.
\newblock {Appliances Energy Prediction}.
\newblock UCI Machine Learning Repository, 2017.
\newblock {DOI}: 10.24432/C5VC8G.

\bibitem[Chen et~al.(2018)Chen, Rubanova, Bettencourt, and
  Duvenaud]{chen2018neural}
Ricky~TQ Chen, Yulia Rubanova, Jesse Bettencourt, and David~K Duvenaud.
\newblock Neural ordinary differential equations.
\newblock In \emph{Advances in Neural Information Processing Systems},
  volume~31, 2018.

\bibitem[Damianou and Lawrence(2013)]{damianou2013deep}
Andreas Damianou and Neil~D. Lawrence.
\newblock Deep {G}aussian processes.
\newblock In \emph{International Conference on Artificial Intelligence and
  Statistics}, pages 207--215, 2013.

\bibitem[De~Souza et~al.(2023)De~Souza, Nikitin, John, Ross, {\'A}lvarez,
  Deisenroth, Gomes, Mesquita, and Mattos]{de2023thin}
Daniel~Augusto De~Souza, Alexander Nikitin, ST~John, Magnus Ross, Mauricio~A
  {\'A}lvarez, Marc Deisenroth, Jo{\~a}o~Paulo Gomes, Diego Mesquita, and
  C{\'e}sar~Lincoln Mattos.
\newblock Thin and deep {G}aussian processes.
\newblock In \emph{Advances in Neural Information Processing Systems},
  volume~36, 2023.

\bibitem[Dunlop et~al.(2018)Dunlop, Girolami, Stuart, and
  Teckentrup]{dunlop2018howdeep}
Matthew~M Dunlop, Mark~A Girolami, Andrew~M Stuart, and Aretha~L Teckentrup.
\newblock How deep are deep {G}aussian processes?
\newblock \emph{Journal of Machine Learning Research}, 19\penalty0
  (54):\penalty0 1--46, 2018.

\bibitem[Duvenaud et~al.(2014)Duvenaud, Rippel, Adams, and
  Ghahramani]{duvenaud2014deep}
David Duvenaud, Oren Rippel, Ryan Adams, and Zoubin Ghahramani.
\newblock Avoiding pathologies in very deep networks.
\newblock In \emph{International Conference on Artificial Intelligence and
  Statistics}, pages 202--210, 2014.

\bibitem[Fanaee-T(2013)]{bike_sharing_275}
Hadi Fanaee-T.
\newblock Bike sharing.
\newblock UCI Machine Learning Repository, 2013.
\newblock {DOI}: 10.24432/C5W894.

\bibitem[Friedman and Nachman(2000)]{friedman2000networks}
Nir Friedman and Iftach Nachman.
\newblock {Gaussian} process networks.
\newblock In \emph{Proceedings of the Sixteenth Conference on Uncertainty in
  Artificial Intelligence}, 2000.

\bibitem[Hegde et~al.(2019)Hegde, Heinonen, L{\"a}hdesm{\"a}ki, and
  Kaski]{hegde2019deep}
Pashupati Hegde, Markus Heinonen, Harri L{\"a}hdesm{\"a}ki, and Samuel Kaski.
\newblock Deep learning with differential {G}aussian process flows.
\newblock In \emph{International Conference on Artificial Intelligence and
  Statistics}, pages 1--15, 2019.

\bibitem[Le et~al.(2005)Le, Smola, and Canu]{le2005heteroscedastic}
Quoc~V Le, Alex~J Smola, and St{\'e}phane Canu.
\newblock Heteroscedastic {G}aussian process regression.
\newblock In \emph{Proceedings of the 22nd international conference on Machine
  learning}, pages 489--496, 2005.

\bibitem[Li et~al.(2020)Li, Wong, Chen, and Duvenaud]{li2020scalable}
Xuechen Li, Ting-Kam~Leonard Wong, Ricky~TQ Chen, and David Duvenaud.
\newblock Scalable gradients for stochastic differential equations.
\newblock In \emph{International conference on artificial intelligence and
  statistics}, pages 3870--3882. PMLR, 2020.

\bibitem[Maro{\~n}as et~al.(2021)Maro{\~n}as, Hamelijnck, Knoblauch, and
  Damoulas]{maronas2021transforming}
Juan Maro{\~n}as, Oliver Hamelijnck, Jeremias Knoblauch, and Theodoros
  Damoulas.
\newblock Transforming {Gaussian} processes with normalizing flows.
\newblock In \emph{Proceedings of the 24th International Conference on
  Artificial Intelligence and Statistics}, volume 130, pages 1081--1089. PMLR,
  2021.

\bibitem[Perdikaris et~al.(2017)Perdikaris, Raissi, Damianou, Lawrence, and
  Karniadakis]{perdikaris2017nonlinear}
Paris Perdikaris, Maziar Raissi, Andreas Damianou, Neil~D. Lawrence, and
  George~Em Karniadakis.
\newblock Nonlinear information fusion algorithms for data-efficient
  multi-fidelity modelling.
\newblock \emph{Proceedings of the Royal Society A}, 473\penalty0
  (2198):\penalty0 20160751, 2017.

\bibitem[Rasmussen and Williams(2006)]{rasmussen2006gaussian}
Carl~Edward Rasmussen and Christopher K.~I. Williams.
\newblock \emph{Gaussian Processes for Machine Learning}.
\newblock MIT Press, Cambridge, MA, USA, 2006.

\bibitem[Salimbeni and Deisenroth(2017)]{salimbeni2017doubly}
Hugh Salimbeni and Marc Deisenroth.
\newblock Doubly stochastic variational inference for deep {G}aussian
  processes.
\newblock In \emph{Advances in Neural Information Processing Systems},
  volume~30, 2017.

\bibitem[Vijayakumar and Schaal(2000)]{vijayakumarn}
S~Vijayakumar and S~Schaal.
\newblock An o(n) algorithm for incremental real time learning in high
  dimensional space.
\newblock In \emph{Proc. of the Seventeenth International Conference on Machine
  Learning}, pages 1079--1086, 2000.

\bibitem[Wilson et~al.(2016)Wilson, Hu, Salakhutdinov, and
  Xing]{wilson2016deep}
Andrew~Gordon Wilson, Zhiting Hu, Ruslan Salakhutdinov, and Eric~P Xing.
\newblock Deep kernel learning.
\newblock In \emph{International Conference on Artificial Intelligence and
  Statistics}, pages 370--378, 2016.

\bibitem[Wyrwal et~al.(2025)Wyrwal, Krause, and Borovitskiy]{wyrwalresidual}
Kacper Wyrwal, Andreas Krause, and Viacheslav Borovitskiy.
\newblock Residual deep {G}aussian processes on manifolds.
\newblock In \emph{The Thirteenth International Conference on Learning
  Representations}, 2025.

\bibitem[Zhao et~al.(2021)Zhao, Emzir, and S{\"a}rkk{\"a}]{zhao2021deep}
Zheng Zhao, Muhammad Emzir, and Simo S{\"a}rkk{\"a}.
\newblock Deep state-space {Gaussian} processes.
\newblock \emph{Statistics and Computing}, 2021.

\end{thebibliography}

\clearpage

\appendix
\section{\texorpdfstring{Proof of
Proposition~\ref{prop:lipschitz_explosion}}{Proof of Proposition~}}\label{proof-of-proposition}

\label{app:optimization_stability}

Recall that the single-sample DGP objective is

\begin{align*}
\hat{\mathcal L}_{\mathrm{DGP}}(\boldsymbol{\theta})
=
\log p(\mathbf y \mid \mathbf f^{(L)})
-
\sum_{\ell=1}^{L}\mathrm{KL}\!\bigl[\cdot \,\Vert\, \cdot\bigr].
\end{align*}

Let $L_{\mathrm{DGP}}$ denote a Lipschitz constant of
$\hat{\mathcal L}_{\mathrm{DGP}}(\boldsymbol{\theta})$. By standard
smoothness arguments,

\begin{align*}
L_{\mathrm{DGP}}
\le
\sup_{\boldsymbol{\theta}}
\left\|
\nabla_{\boldsymbol{\theta}}
\hat{\mathcal L}_{\mathrm{DGP}}(\boldsymbol{\theta})
\right\|.
\end{align*}

Using the decomposition of the objective and the triangle inequality, we
obtain

\begin{align*}
L_{\mathrm{DGP}}
\le
L_{\mathrm{data}} + L_{\mathrm{KL}},
\end{align*}

where

\begin{align*}
L_{\mathrm{data}}
\coloneqq
\sup_{\boldsymbol{\theta}}
\left\|
\nabla_{\boldsymbol{\theta}}
\log p(\mathbf y \mid \mathbf f^{(L)})
\right\|,
\qquad
L_{\mathrm{KL}}
\coloneqq
\sup_{\boldsymbol{\theta}}
\left\|
\nabla_{\boldsymbol{\theta}}
\sum_{\ell=1}^{L}\mathrm{KL}\!\bigl[\cdot \,\Vert\, \cdot\bigr]
\right\|.
\end{align*}

We bound these two terms separately.

\begin{lemma}\label{lem:data_term_dgp}
Assume that

\begin{align*}
\left\|
\frac{\partial \log p(\mathbf y \mid \mathbf f^{(L)})}{\partial \mathbf h^{(L)}}
\right\|
\le C_{\mathrm{out}},
\qquad
\left\|
\frac{\partial f^{(k)}(\mathbf h^{(k-1)})}{\partial \boldsymbol{\theta}^{(k)}}
\right\|
\le C_{\mathrm{param}},
\end{align*}

for all $k=1,\dots,L$, and that each sampled layer function is
$\lambda_\ell$-Lipschitz with respect to its input, uniformly over
$\boldsymbol{\theta}$, i.e.

\begin{align*}
\left\|
\frac{\partial f^{(\ell)}(\mathbf h^{(\ell-1)})}{\partial \mathbf h^{(\ell-1)}}
\right\|
\le
\lambda_\ell,
\qquad \ell=1,\dots,L.
\end{align*}

Then

\begin{align*}
L_{\mathrm{data}}
\le
C
\left(
\sum_{k=1}^{L}\prod_{\ell=k+1}^{L}\lambda_\ell^2
\right)^{1/2},
\end{align*}

where $C \coloneqq C_{\mathrm{out}}C_{\mathrm{param}}$ is independent
of $L$.

\end{lemma}

\begin{proof}

Write

\begin{align*}
\boldsymbol{\theta}
=
\bigl(
\boldsymbol{\theta}^{(1)},\dots,\boldsymbol{\theta}^{(L)}
\bigr),
\end{align*}

where $\boldsymbol{\theta}^{(k)}$ is the parameter block associated
with layer $k$. For each $k\in\{1,\dots,L\}$, by the chain rule,

\begin{align*}
\frac{\partial \log p(\mathbf y \mid \mathbf f^{(L)})}{\partial \boldsymbol{\theta}^{(k)}}
=
\frac{\partial \log p(\mathbf y \mid \mathbf f^{(L)})}{\partial \mathbf h^{(L)}}
\frac{\partial \mathbf h^{(L)}}{\partial \boldsymbol{\theta}^{(k)}}.
\end{align*}

Since the latent states satisfy

\begin{align*}
\mathbf h^{(\ell)} = f^{(\ell)}\!\bigl(\mathbf h^{(\ell-1)}\bigr)+\boldsymbol{\epsilon}^{(\ell)},
\end{align*}

and the additive noise $\boldsymbol{\epsilon}^{(\ell)}$ does not
depend on $\boldsymbol{\theta}$, we have

\begin{align*}
\frac{\partial \mathbf h^{(\ell)}}{\partial \mathbf h^{(\ell-1)}}
=
\frac{\partial f^{(\ell)}(\mathbf h^{(\ell-1)})}{\partial \mathbf h^{(\ell-1)}},
\qquad
\frac{\partial \mathbf h^{(k)}}{\partial \boldsymbol{\theta}^{(k)}}
=
\frac{\partial f^{(k)}(\mathbf h^{(k-1)})}{\partial \boldsymbol{\theta}^{(k)}}.
\end{align*}

Therefore,

\begin{align*}
\frac{\partial \mathbf h^{(L)}}{\partial \boldsymbol{\theta}^{(k)}}
=
\left(
\prod_{\ell=k+1}^{L}
\frac{\partial \mathbf h^{(\ell)}}{\partial \mathbf h^{(\ell-1)}}
\right)
\frac{\partial \mathbf h^{(k)}}{\partial \boldsymbol{\theta}^{(k)}}
=
\left(
\prod_{\ell=k+1}^{L}
\frac{\partial f^{(\ell)}(\mathbf h^{(\ell-1)})}{\partial \mathbf h^{(\ell-1)}}
\right)
\frac{\partial f^{(k)}(\mathbf h^{(k-1)})}{\partial \boldsymbol{\theta}^{(k)}}.
\end{align*}

Taking norms and using sub-multiplicativity of operator norms yields

\begin{align*}
\left\|
\frac{\partial \log p(\mathbf y \mid \mathbf f^{(L)})}{\partial \boldsymbol{\theta}^{(k)}}
\right\|
\le
\left\|
\frac{\partial \log p(\mathbf y \mid \mathbf f^{(L)})}{\partial \mathbf h^{(L)}}
\right\|
\prod_{\ell=k+1}^{L}
\left\|
\frac{\partial f^{(\ell)}(\mathbf h^{(\ell-1)})}{\partial \mathbf h^{(\ell-1)}}
\right\|
\left\|
\frac{\partial f^{(k)}(\mathbf h^{(k-1)})}{\partial \boldsymbol{\theta}^{(k)}}
\right\|.
\end{align*}

By the assumptions,

\begin{align*}
\left\|
\frac{\partial \log p(\mathbf y \mid \mathbf f^{(L)})}{\partial \boldsymbol{\theta}^{(k)}}
\right\|
\le
C_{\mathrm{out}}C_{\mathrm{param}}
\prod_{\ell=k+1}^{L}\lambda_\ell.
\end{align*}

Hence,

\begin{align*}
\left\|
\nabla_{\boldsymbol{\theta}}
\log p(\mathbf y \mid \mathbf f^{(L)})
\right\|
=
\left(
\sum_{k=1}^{L}
\left\|
\nabla_{\boldsymbol{\theta}^{(k)}}
\log p(\mathbf y \mid \mathbf f^{(L)})
\right\|^2
\right)^{1/2}
\le
C_{\mathrm{out}}C_{\mathrm{param}}
\left(
\sum_{k=1}^{L}\prod_{\ell=k+1}^{L}\lambda_\ell^2
\right)^{1/2}.
\end{align*}

Taking the supremum over $\boldsymbol{\theta}$ gives the result.

\end{proof}

\begin{lemma}\label{lem:kl_term_dgp}
Under the layerwise variational parameterization of the DGP objective,
the KL regularizer decomposes as

\begin{align*}
\sum_{\ell=1}^{L}\mathrm{KL}\!\bigl[\cdot \,\Vert\, \cdot\bigr]
=
\sum_{\ell=1}^{L}\mathrm{KL}_\ell(\boldsymbol{\theta}^{(\ell)}).
\end{align*}

If each layerwise KL term has uniformly bounded gradient, namely

\begin{align*}
\sup_{\boldsymbol{\theta}^{(\ell)}}
\left\|
\nabla_{\boldsymbol{\theta}^{(\ell)}}
\mathrm{KL}_\ell(\boldsymbol{\theta}^{(\ell)})
\right\|
\le
C_{\mathrm{KL}},
\qquad \ell=1,\dots,L,
\end{align*}

then

\begin{align*}
L_{\mathrm{KL}}(L)=O(\sqrt{L}).
\end{align*}

\end{lemma}

\begin{proof}

Since the KL regularizer decomposes layerwise, we have

\begin{align*}
\nabla_{\boldsymbol{\theta}}
\sum_{\ell=1}^{L}\mathrm{KL}_\ell(\boldsymbol{\theta}^{(\ell)})
=
\bigl(
\nabla_{\boldsymbol{\theta}^{(1)}}\mathrm{KL}_1,\dots,
\nabla_{\boldsymbol{\theta}^{(L)}}\mathrm{KL}_L
\bigr).
\end{align*}

Therefore,

\begin{align*}
\left\|
\nabla_{\boldsymbol{\theta}}
\sum_{\ell=1}^{L}\mathrm{KL}_\ell(\boldsymbol{\theta}^{(\ell)})
\right\|
=
\left(
\sum_{\ell=1}^{L}
\left\|
\nabla_{\boldsymbol{\theta}^{(\ell)}}
\mathrm{KL}_\ell(\boldsymbol{\theta}^{(\ell)})
\right\|^2
\right)^{1/2}
\le
\left(
\sum_{\ell=1}^{L} C_{\mathrm{KL}}^2
\right)^{1/2}
=
\sqrt{L}\,C_{\mathrm{KL}}.
\end{align*}

Taking the supremum over $\boldsymbol{\theta}$ yields

\begin{align*}
L_{\mathrm{KL}}
\le
\sqrt{L}\,C_{\mathrm{KL}},
\end{align*}

and hence $L_{\mathrm{KL}}(L) = O(\sqrt{L})$.

\end{proof}

Combining Lemmas\textasciitilde{}\ref{lem:data_term_dgp}
and\textasciitilde{}\ref{lem:kl_term_dgp}, we obtain

\begin{align*}
L_{\mathrm{DGP}}
\le
C
\left(
\sum_{k=1}^{L}\prod_{\ell=k+1}^{L}\lambda_\ell^2
\right)^{1/2}
+
O(\sqrt{L}),
\end{align*}

where the $O(\sqrt{L})$ term captures the contribution of the KL
regularizer. This proves
Proposition~\ref{prop:lipschitz_explosion}.

\section{\texorpdfstring{Proof of
Theorem~\ref{thm:pathwise_stability}}{Proof of Theorem~}}\label{proof-of-theorem}

\label{app:proof_thm_3.1}

\begin{proof}[Proof of Theorem~\ref{thm:pathwise_stability}]
Define
\begin{align}
    \bm{\Delta}_t = \mathbf{h}_t^{(1)} - \mathbf{h}_t^{(2)},
\end{align}
so that $\bm{\Delta}_0=\mathbf{0}$.
Subtracting the two SDEs driven by the same Brownian motion gives
\begin{align}
    d\bm{\Delta}_t
    &=
    \big[ f(\mathbf{h}_t^{(1)}, t; \bm{\theta}_1) - f(\mathbf{h}_t^{(2)}, t; \bm{\theta}_2) \big]\,dt
    + (\sigma_t-\sigma_t)\,d\mathbf{w}_t.
\end{align}
Hence the stochastic terms cancel exactly and
\begin{align}
    \bm{\Delta}_t
    = \int_0^t \big[ f(\mathbf{h}_s^{(1)}, s; \bm{\theta}_1) - f(\mathbf{h}_s^{(2)}, s; \bm{\theta}_2) \big]\,ds.
\end{align}

Adding and subtracting $f(\mathbf{h}_s^{(2)}, s; \bm{\theta}_1)$ yields
\begin{align}
    \bm{\Delta}_t
    =
    \int_{0}^{t} \big[ f(\mathbf{h}_s^{(1)}, s; \bm{\theta}_1)
    - f(\mathbf{h}_s^{(2)}, s; \bm{\theta}_1) \big]
    +
    \big[ f(\mathbf{h}_s^{(2)}, s; \bm{\theta}_1)
    - f(\mathbf{h}_s^{(2)}, s; \bm{\theta}_2) \big] ds .
\end{align}
Taking norms and using the Lipschitz assumptions,
\begin{align*}
\big\| f(\mathbf{h}_s^{(1)}, s; \bm{\theta}_1) - f(\mathbf{h}_s^{(2)}, s; \bm{\theta}_1) \big\|
&\le L_h \|\bm{\Delta}_s\|,\\
\big\| f(\mathbf{h}_s^{(2)}, s; \bm{\theta}_1) - f(\mathbf{h}_s^{(2)}, s; \bm{\theta}_2) \big\|
&\le L_\theta \|\bm{\theta}_1-\bm{\theta}_2\|.
\end{align*}
Therefore, with $u(t)=\|\bm{\Delta}_t\|$ and $\delta_\theta=\|\bm{\theta}_1-\bm{\theta}_2\|$,
\begin{align}
\label{eq:integral_gronwall}
    u(t) \le \int_0^t \left( L_h u(s) + L_\theta \delta_\theta \right) \, ds.
\end{align}

Define
\begin{align}
    y(t) = \int_0^t \left( L_h u(s) + L_\theta \delta_\theta \right) \, ds,
\end{align}
so that $u(t)\le y(t)$ and $y(0)=0$.
Then
\begin{align}
    \frac{d y(t)}{d t}
    &= L_h u(t) + L_\theta \delta_\theta \nonumber\\
    &\le L_h y(t) + L_\theta \delta_\theta.
\end{align}
Multiplying by $e^{-L_h t}$ gives
\begin{align}
    \frac{d}{dt}\!\left[e^{-L_h t}y(t)\right]
    \le L_\theta e^{-L_h t}\delta_\theta.
\end{align}
Integrating from $0$ to $t$ and using $y(0)=0$, we obtain, for $L_h>0$,
\begin{align}
    e^{-L_h t} y(t)
    &\le
    L_\theta \delta_\theta \int_0^t e^{-L_h s}\,ds
    =
    \frac{L_\theta \delta_\theta}{L_h} (1 - e^{-L_h t}),
\end{align}
and hence
\begin{align}
    y(t)
    \le
    \frac{L_\theta}{L_h}(e^{L_h t}-1)\delta_\theta.
\end{align}
If $L_h=0$, Eq.~\eqref{eq:integral_gronwall} reduces to
\begin{align}
    u(t)\le \int_0^t L_\theta \delta_\theta\,ds = L_\theta t\,\delta_\theta.
\end{align}
Combining the two cases,
\begin{align}
    \|\mathbf{h}_t^{(1)} - \mathbf{h}_t^{(2)}\|
    \le
    L_\theta\,\phi(L_h,t)\,\|\bm{\theta}_1-\bm{\theta}_2\|,
\end{align}
where
\begin{align}
    \phi(L_h,t)
    =
    \begin{cases}
    \dfrac{e^{L_h t}-1}{L_h}, & L_h > 0,\\[0.4em]
    t, & L_h = 0.
    \end{cases}
\end{align}
\end{proof}

\section{\texorpdfstring{Proof of
Corollary~\ref{cor:contractive_stability}}{Proof of Corollary~}}\label{proof-of-corollary}

\label{app:proof_cor}

\begin{proof}

Starting from the pathwise ODE for $\bm{\Delta}_t$ established in the
proof of Theorem~\ref{thm:pathwise_stability}, we have

\begin{align*}
  \frac{1}{2}\frac{d}{dt}\|\bm{\Delta}_t\|^2
  =
  \bigl\langle
    \bm{\Delta}_t,\,
    f(\mathbf{h}_t^{(1)},t;\bm{\theta}_1)
    -
    f(\mathbf{h}_t^{(2)},t;\bm{\theta}_2)
  \bigr\rangle.
\end{align*}

Adding and subtracting $f(\mathbf{h}_t^{(2)},t;\bm{\theta}_1)$ yields

\begin{align*}
  \frac{1}{2}\frac{d}{dt}\|\bm{\Delta}_t\|^2
  =
  \underbrace{
    \bigl\langle
      \bm{\Delta}_t,\,
      f(\mathbf{h}_t^{(1)},t;\bm{\theta}_1)
      -
      f(\mathbf{h}_t^{(2)},t;\bm{\theta}_1)
    \bigr\rangle
  }_{(\mathrm{I})}
  +
  \underbrace{
    \bigl\langle
      \bm{\Delta}_t,\,
      f(\mathbf{h}_t^{(2)},t;\bm{\theta}_1)
      -
      f(\mathbf{h}_t^{(2)},t;\bm{\theta}_2)
    \bigr\rangle
  }_{(\mathrm{II})}.
\end{align*}

By the one-sided contractivity assumption,
$(\mathrm{I}) \le -\mu\|\bm{\Delta}_t\|^2$. For $(\mathrm{II})$,
Cauchy--Schwarz and the Lipschitz condition in $\bm{\theta}$ give

\begin{align*}
  (\mathrm{II})
  \le
  \|\bm{\Delta}_t\|\,L_\theta\|\bm{\theta}_1-\bm{\theta}_2\|.
\end{align*}

Applying Young's inequality with parameter $\mu$,

\begin{align*}
  \|\bm{\Delta}_t\|\,L_\theta\|\bm{\theta}_1-\bm{\theta}_2\|
  \le
  \frac{\mu}{2}\|\bm{\Delta}_t\|^2
  +
  \frac{L_\theta^2}{2\mu}\|\bm{\theta}_1-\bm{\theta}_2\|^2,
\end{align*}

and combining the two bounds yields

\begin{align*}
  \frac{d}{dt}\|\bm{\Delta}_t\|^2
  \le
  -\mu\|\bm{\Delta}_t\|^2
  +
  \frac{L_\theta^2}{\mu}\|\bm{\theta}_1-\bm{\theta}_2\|^2.
\end{align*}

Since $\bm{\Delta}_0=\mathbf{0}$, Gr\"onwall's inequality gives

\begin{align*}
  \|\bm{\Delta}_t\|^2
  \le
  \frac{L_\theta^2}{\mu^2}
  \bigl(1-e^{-\mu t}\bigr)
  \|\bm{\theta}_1-\bm{\theta}_2\|^2.
\end{align*}

Taking square roots and evaluating at $t=T$,

\begin{align*}
  \|\mathbf{h}_T^{(1)}-\mathbf{h}_T^{(2)}\|
  \le
  \frac{L_\theta}{\mu}\sqrt{1-e^{-\mu T}}\,
  \|\bm{\theta}_1-\bm{\theta}_2\|
  \le
  \frac{L_\theta}{\mu}\|\bm{\theta}_1-\bm{\theta}_2\|,
\end{align*}

for almost every sample path.

\end{proof}

\section{Complexity Analysis}

\label{app:complexity}

In CDGP, the inducing set is the Cartesian product
$\mathbf{Z}_h \times \mathbf{Z}_t$ of size $M_h M_t$. Exploiting the
Kronecker structure of the separable kernel, the per-iteration
complexity is $\mathcal{O}(M_h^3 + M_t^3 + NJ M_h M_t)$.

In FlowDGP, inference is performed only over the base GP $f_0$ with
$M$ inducing points, giving a per-iteration complexity of
$\mathcal{O}(M^3 + NM^2 + NJC_g)$, where $C_g$ is the cost of
evaluating $g_{\bm{\phi}}$. The reduction from $M_h M_t$ to $M$
inducing points more than compensates for the additional ODE solver
cost, resulting in substantially faster training in practice.

This complexity gap is reflected in the benchmark train-time results in
Table~\ref{tab:results}, where FlowDGP is already
substantially faster than CDGP on the smaller benchmark subsets. That
empirical gap is the main reason why CDGP and residual DGP are not
carried into the full-size large-scale experiment.

\section{Simulation Details}

\label{app:simutation_detail}

The experiments were run on a single NVIDIA RTX 4500 Ada Generation GPU
(24\,GB).   Note that, for each test input $\x_n$, we draw $S = 300$ Monte Carlo samples
$\{f_s(\x_n)\}_{s=1}^{S}$ from the model and form a Gaussian predictive 
distribution $\mathcal{N}(\mu_n, \sigma_n^2)$ whose moments are estimated 
from the samples:
\begin{equation}
    \mu_n = \frac{1}{S}\sum_{s=1}^{S} f_s(\x_n), 
    \qquad 
    \sigma_n^2 = \frac{1}{S}\sum_{s=1}^{S}\bigl(f_s(\x_n) - \mu_n\bigr)^2 
    + \sigma_{\text{noise}}^2.
\end{equation}
The negative log-likelihood is then evaluated as
\begin{equation}
    \mathrm{NLL} = \frac{1}{2N}\sum_{n=1}^{N}\left[
        \log(2\pi\sigma_n^2) + \frac{(y_n - \mu_n)^2}{\sigma_n^2}
    \right].
\end{equation}

\subsection{Datasets Description}

\label{app:dataset}

We evaluate on a one-dimensional toy example, three synthetic
benchmarks, three real-world regression tasks at two sample sizes,
and YearPredictionMSD. For the toy and benchmark datasets,
inputs are standardised per dimension and targets are standardised by
the training-set mean and standard deviation. The toy example is used
in Section~\ref{subsec:toy_example} with $N_{\mathrm{train}} = 300$
and $N_{\mathrm{test}} = 500$. For the benchmark experiments
(Section~\ref{subsec:benchmark}), each dataset uses
$N_{\mathrm{train}} = 300$ and $N_{\mathrm{test}} = 500$; for the
real-world tasks, we randomly subsample from the full datasets to this
size. For the large-scale experiments
(Section~\ref{subsec:large-scale}), the full versions of the real-world
datasets are used.

\paragraph{1D toy example (\texttt{toy\_example}).}
Inputs $x \in [0, 1]$ are drawn uniformly, and the scalar target is
$y = f(x) + \sigma(x)\,\varepsilon$ with
$\varepsilon \sim \mathcal{N}(0, 1)$. Writing
$\tilde{x} = (x - x_{\min})/(x_{\max} - x_{\min})$ for the normalized
input, the clean signal is a superposition of seven components, $f(x) = \sum_{i=1}^{7} f_i(\tilde{x})$, where
\begin{align*}
  f_1(\tilde{x}) &=
    \bigl[0.08 + 0.42\,\mathrm{sigmoid}(22(\tilde{x} - 0.28))
    + 0.82\,\tilde{x}^2\bigr]\sin\phi(\tilde{x}),
    \\
  f_2(\tilde{x}) &=
    0.12\sin(2\phi(\tilde{x}) + 6\tilde{x})
    + 0.10\sin(92\,\tilde{x}^2)
    \\
    &\qquad
    + 0.08\sin\!\bigl(36\tilde{x} + 2.2\sin(4\pi\tilde{x})\bigr)
    + 0.10(\tilde{x} - 0.48)^2,
    \\
  f_3(\tilde{x}) &=
    0.42\sin(2.2\pi\tilde{x}) + 0.38\sin(56\pi\tilde{x}),
    \\
  f_4(\tilde{x}) &=
    0.32\,\exp\!\bigl(-\tfrac{1}{2}(\tilde{x} - 0.5)^2 / 0.11^2\bigr)
    \sin(84\pi\tilde{x}),
    \\
  f_5(\tilde{x}) &=
    0.28\sqrt{\lvert \tilde{x} - 0.34 \rvert},
    \\
  f_6(\tilde{x}) &=
    0.45\,\mathrm{sigmoid}\!\bigl(110(\tilde{x} - 0.58)\bigr),
    \\
  f_7(\tilde{x}) &=
    0.32\bigl[
      w(\tilde{x})\,
      \sin\!\bigl(9\tilde{x}(0.65 + 0.35\tilde{x})\bigr)
      \\
      &\qquad
      + (1 - w(\tilde{x}))\sin(52\tilde{x})
    \bigr],
\end{align*}
with chirp phase
$\phi(\tilde{x}) = \tfrac{\omega_{\mathrm{hi}}}{\beta}
\bigl(1 - e^{-\beta\tilde{x}}\bigr) + \omega_{\mathrm{lo}}\,\tilde{x}$
(parameters $\omega_{\mathrm{hi}} = 88$, $\omega_{\mathrm{lo}} = 5$,
$\beta = 9.5$) and regime-switching gate
$w(\tilde{x}) = \mathrm{sigmoid}(28(\tilde{x} - 0.45))$. The seven
components respectively contribute: (1) a chirp with growing envelope,
(2) second-harmonic, quadratic-phase, and FM terms with a mild
polynomial trend, (3) spectrally conflicting slow and fast sinusoids,
(4) a localised high-frequency burst inside a narrow Gaussian window,
(5) a square-root cusp, (6) a near-discontinuous step, and (7) a regime
blend whose dominant frequency differs sharply between the two sides of
the domain. Noise is heteroscedastic with
$\sigma(x) = 0.16\bigl(0.55 + 0.9\,\tilde{x}\bigr)$, so the output
variance grows roughly $2.6\times$ from left to right. We use
$N_{\mathrm{train}} = 300$ and $N_{\mathrm{test}} = 500$; inputs and
targets are standardised using training-set statistics. Although the
task is strictly one-dimensional ($D_x = 1$), the combination of
multi-scale sinusoids, a localised burst, a cusp, a near-step, and a
regime change makes it adversarial to any model with a single
stationary lengthscale.

\paragraph{Heteroscedastic (\texttt{heteroscedastic}).}
Inputs $\x \in \mathbb{R}^{D_x}$ with $D_x = 10$ are drawn uniformly on
$[-4, 4]^{D_x}$. The latent signal is
\begin{equation*}
    g(\x) = \sin(2x_1) + 0.5\cos(x_2) + 0.3 x_3,
    \qquad
    f_k(\x) = g(\x) + 0.2\sin(k x_1),
    \quad k = 1, \dots, 10,
\end{equation*}
and intermediate observations are
$\tilde{y}_k = f_k(\x) + \sigma(\x)\,\varepsilon_k$ with
$\varepsilon_k \sim \mathcal{N}(0, 1)$ and input-dependent noise scale
\begin{equation*}
    \sigma(\x) = 0.1 + 0.5\,\mathrm{sigmoid}(1.5 x_1 - 0.5 x_2).
\end{equation*}
The scalar target is $y = \frac{1}{10}\sum_{k=1}^{10} \tilde{y}_k$, so
the effective noise is heteroscedastic with roughly a $5\times$
variance ratio across the input space. This benchmark evaluates a
model's ability to fit input-dependent aleatoric uncertainty.

\paragraph{Nonstationary (\texttt{nonstationary}).}
Inputs $\x \in \mathbb{R}^{D_x}$ with $D_x = 10$ are drawn uniformly on
$[-4, 4]^{D_x}$. The clean signal is a smooth two-regime blend governed
by the first coordinate,
\begin{equation*}
    f(\x) = \bigl(1 - w(x_1)\bigr)\sin(x_1) + w(x_1)\cos(x_1) + 0.2\,x_2,
    \qquad
    w(x_1) = \mathrm{sigmoid}(x_1),
\end{equation*}
with homoscedastic Gaussian noise
$\varepsilon \sim \mathcal{N}(0, 0.2^2)$. The change in effective
lengthscale between the $x_1 \!\ll\! 0$ and $x_1 \!\gg\! 0$ regimes is
the principal challenge: a single stationary kernel cannot match both
regimes simultaneously.

\paragraph{Heavy-tailed (\texttt{heavy\_tailed}).}
Inputs $\x \in \mathbb{R}^{D_x}$ with $D_x = 10$ are drawn uniformly on
$[-5, 5]^{D_x}$, and the latent signal takes the additive form
\begin{equation*}
    f(\x) = \frac{1}{D_x}\sum_{i=1}^{D_x}
    \bigl(\sin(1.5 x_i) + 0.3 x_i\bigr).
\end{equation*}
Observations are $y = f(\x) + \varepsilon$ with Student-$t$ noise
$\varepsilon \sim t_{\nu}(0, s)$, where $\nu = 3$ and $s = 0.3$. The
fat tails induce a non-negligible fraction of gross outliers, which
stresses baselines that assume Gaussian noise.

\paragraph{SARCOS (\texttt{sarcos}, \texttt{sarcos\_small}).}
The SARCOS inverse-dynamics benchmark~\citep{vijayakumarn}: given
$D_x = 21$ joint-state features (seven positions, seven velocities,
seven accelerations) from a seven-degree-of-freedom anthropomorphic
robot arm, predict the first joint torque. We use the standard split
with $44{,}484$ training and $4{,}449$ test samples in the large-scale
setting, and subsample to $N_{\mathrm{train}} = 300$ and
$N_{\mathrm{test}} = 500$ in the benchmark setting
(\texttt{sarcos\_small}). Features and targets are standardised using
the full parent training set.

\paragraph{UCI Bike (\texttt{uci\_bike}, \texttt{uci\_bike\_small}).}
The UCI Bike Sharing dataset~\citep{bike_sharing_275}: the target is
the hourly rental count (\texttt{cnt}), and the inputs are
$D_x = 12$ weather and calendar features (\texttt{season},
\texttt{yr}, \texttt{mnth}, \texttt{hr}, \texttt{holiday},
\texttt{weekday}, \texttt{workingday}, \texttt{weathersit},
\texttt{temp}, \texttt{atemp}, \texttt{hum}, \texttt{windspeed}). We
apply an $80/20$ time-ordered split with per-feature $z$-scoring and
target standardisation, yielding $13{,}903$ training and $3{,}476$ test
samples in the large-scale setting; the benchmark setting subsamples to
$N_{\mathrm{train}} = 300$ and $N_{\mathrm{test}} = 500$. The target is
highly non-Gaussian (count data with mixed diurnal and seasonal
regimes), making this a useful test of non-stationarity.

\paragraph{UCI Energy (\texttt{uci\_energy}, \texttt{uci\_energy\_small}).}
The UCI Appliances Energy Prediction
dataset~\citep{appliances_energy_prediction_374}: the target is
appliance energy consumption (Wh), and the inputs are $D_x = 29$ indoor
sensor readings and weather variables, augmented with four derived
temporal features (hour of day, day of week, month, weekend indicator)
parsed from the timestamp. Metadata columns (\texttt{date},
\texttt{rv1}, \texttt{rv2}) and the target itself are excluded from the
inputs. We apply a chronological $80/20$ split with $z$-scoring,
yielding $15{,}788$ training and $3{,}947$ test samples in the
large-scale setting; the benchmark setting subsamples to
$N_{\mathrm{train}} = 300$ and $N_{\mathrm{test}} = 500$. Targets are
strongly right-skewed with sparse high-activity spikes, producing
effectively heavy-tailed residuals in small samples.

\paragraph{YearPredictionMSD.}
\label{app:yearprediction}
YearPredictionMSD~\citep{yearpredictionmsd} is a song-release-year
prediction dataset containing $515{,}345$ observations, each represented
by 90 numerical audio features. The task is scalar regression from audio
characteristics to release year.

\paragraph{Summary.}
Table~\ref{tab:datasets} summarises the input dimension, sample sizes
(benchmark setting), and principal regression challenge targeted by
each dataset. All tasks are scalar regression ($D_y = 1$); sample sizes
for the corresponding large-scale datasets (\texttt{sarcos},
\texttt{uci\_bike}, \texttt{uci\_energy}) are given in the paragraphs
above.

\begin{table}[h]
\centering
\small
\renewcommand{\arraystretch}{1.15}
\setlength{\tabcolsep}{6pt}
\begin{tabular}{lcccl}
\toprule
Dataset & $D_x$ & $N_{\mathrm{train}}$ & $N_{\mathrm{test}}$
  & Primary challenge \\
\midrule
\texttt{heteroscedastic}     & $10$     & $300$      & $500$
  & input-dependent noise \\
\texttt{nonstationary}       & $10$     & $300$      & $500$
  & two-regime lengthscale change \\
\texttt{heavy\_tailed}       & $10$     & $300$      & $500$
  & Student-$t$ ($\nu\!=\!3$) outliers \\
\midrule
\texttt{sarcos\_small}       & $21$     & $300$      & $500$
  & high-dimensional nonlinear dynamics \\
\texttt{uci\_bike\_small}    & $12$     & $300$      & $500$
  & mixed diurnal/seasonal regimes \\
\texttt{uci\_energy\_small}  & $29$     & $300$      & $500$
  & skewed, spiky real sensor data \\
\midrule
\texttt{sarcos}              & $21$     & $44{,}484$ & $4{,}449$
  & high-dimensional nonlinear dynamics \\
\texttt{uci\_bike}           & $12$     & $13{,}903$ & $3{,}476$
  & mixed diurnal/seasonal regimes \\
\texttt{uci\_energy}         & $29$     & $15{,}788$ & $3{,}947$
  & skewed, spiky real sensor data \\
\bottomrule
\end{tabular}
\caption{Summary of evaluation datasets. The first block lists the
three synthetic benchmarks, the second block the subsampled real-world
datasets used in Section~\ref{subsec:benchmark}, and the third block
their full-size counterparts used in
Section~\ref{subsec:large-scale}. All tasks are scalar regression with
$D_y = 1$.}
\label{tab:datasets}
\end{table}

\subsection{Hyperparameter Settings}

\label{app:setting}

To ensure reproducibility, we provide the hyperparameter configurations
for all methods. Table~\ref{tab:grid-search} lists the
hyperparameters selected via grid search, and
Table~\ref{tab:fixed-hparams} reports the fixed settings
shared across experiments. For each method and dataset, hyperparameters
were selected based on validation performance, considering both NLL and
RMSE. We first identified configurations whose validation NLL was close
to the best-performing one, and among them chose the configuration with
the lowest validation RMSE. This favors models that achieve strong
probabilistic performance without sacrificing point prediction accuracy.
In these tables, $J$ is the number of integration steps and $M_t$ is the
number of inducing locations along depth for CDGP. Neither denotes the
terminal horizon $T$, which is $1.0$ for both CDGP and FlowDGP unless
varied explicitly.

\begin{table}[t]
\centering
\caption{Grid-searched hyperparameters for the benchmark methods. For each method and dataset, we select the configuration with the lowest validation NLL.}
\label{tab:grid-search}
\small
\renewcommand{\arraystretch}{1.15}
\setlength{\tabcolsep}{8pt}
\resizebox{\textwidth}{!}{%
\begin{tabular}{@{}L{2.4cm} L{13.0cm}@{}}
\toprule
\textbf{Method} & \textbf{Search space} \\
\midrule
\rowcolor{black!4}
SparseGP
  & \texttt{lr} $\in\{0.01, 0.02, 0.03, 0.05\}$ \\
HeteroskedasticGP
  & \texttt{lr} $\in\{0.01, 0.02, 0.03, 0.05\}$ \\
\rowcolor{black!4}
DKL
  & \texttt{grid\_size} $\in\{80, 100, 120\}$ \\
\rowcolor{black!4}
  & \texttt{num\_dims} $\in\{2, 3\}$ \\
\rowcolor{black!4}
  & \texttt{lr} $\in\{0.01, 0.02, 0.03, 0.05\}$ \\
TDGP
  & \texttt{m\_v} $\in\{20, 30\}$ \\
  & \texttt{m\_u} $\in\{20, 30\}$ \\
  & \texttt{lr} $\in\{0.01, 0.02, 0.03, 0.05\}$ \\
\rowcolor{black!4}
DGP
  & \texttt{num\_layers} $\in\{2, 3, 5\}$ \\
\rowcolor{black!4}
  & \texttt{lr} $\in\{0.01, 0.02, 0.03, 0.05\}$ \\
ResidualDGP
  & \texttt{num\_layers} $\in\{2, 3, 5\}$ \\
  & \texttt{lr} $\in\{0.01, 0.02, 0.03, 0.05\}$ \\
\rowcolor{black!4}
NODE
  & \texttt{hidden\_dim} $\in\{32, 64\}$ \\
\rowcolor{black!4}
  & \texttt{num\_layers} $\in\{1, 2\}$ \\
\rowcolor{black!4}
  & $J$ $\in\{2, 3, 5\}$ \\
\rowcolor{black!4}
  & \texttt{lr} $\in\{0.01, 0.02, 0.03, 0.05\}$ \\
NSDE
  & \texttt{hidden\_dim} $\in\{32, 64\}$ \\
  & $J$ $\in\{2, 3, 5\}$ \\
  & \texttt{lr} $\in\{0.01, 0.02, 0.03, 0.05\}$ \\
\rowcolor{black!4}
CDGP
  & $J$ $\in\{2, 3, 5\}$ \\
\rowcolor{black!4}
  & $M_t$ $\in\{3, 5\}$ \\
\rowcolor{black!4}
  & \texttt{lr} $\in\{0.01, 0.02, 0.03, 0.05\}$ \\
FlowDGP
  & \texttt{hidden\_dims} $\in\{[32,32], [64,64]\}$ \\
  & $J$ $\in\{2, 3, 5\}$ \\
  & \texttt{lr} $\in\{0.01, 0.02, 0.03, 0.05\}$ \\
\bottomrule
\end{tabular}
}
\end{table}

\begin{table}[t]
\centering
\caption{Fixed hyperparameter settings for the toy and benchmark experiments. Method-specific entries override the defaults.}
\label{tab:fixed-hparams}
\small
\renewcommand{\arraystretch}{1.15}
\setlength{\tabcolsep}{8pt}
\resizebox{\textwidth}{!}{%
\begin{tabular}{@{}L{3.2cm} L{12.2cm}@{}}
\toprule
\textbf{Scope} & \textbf{Setting} \\
\midrule
\rowcolor{black!4}
Defaults
  & \texttt{num\_inducing} $= 30$ (whenever applicable) \\
\rowcolor{black!4}
  & \texttt{epochs} $= 2500$ \\
\rowcolor{black!4}
  & optimiser $=$ Adam \\
DKL
  & \texttt{epochs} $= 500$ \\
\rowcolor{black!4}
CDGP
  & $\sigma = 0.15$ \\
\rowcolor{black!4}
  & $T = 1.0$ (unless otherwise specified) \\
FlowDGP
  & $\sigma = 0.15$ \\
  & $T = 1.0$ (unless otherwise specified) \\
\rowcolor{black!4}
NODE
  & $T = 1.0$ \\
\rowcolor{black!4}
  & \texttt{solver} $=$ Euler \\
\rowcolor{black!4}
  & \texttt{rtol} $= 10^{-4}$,\;
    \texttt{atol} $= 10^{-5}$ \\
\rowcolor{black!4}
  & \texttt{adjoint} $=$ False \\
NSDE
  & $T = 1.0$ \\
  & \texttt{solver} $=$ Euler \\
  & \texttt{activation} $=$ LipSwish \\
\rowcolor{black!4}
TDGP
  & \texttt{diag\_qmu} $=$ True \\ 
  \rowcolor{black!4} & \texttt{epochs=1000} \\ 
TGP
  & \texttt{num\_blocks} $=3$ (SAL) \\
  & \texttt{hidden\_dims} $=[50,50]$ \\
  & \texttt{epochs} $=15000$ \\
  & \texttt{lr} $=0.001$ \\
\bottomrule
\end{tabular}
}
\end{table}

\section{Additional Experimental Comparisons}
\label{app:additional-experiments}

\subsection{Large-scale comparison on YearPredictionMSD}
\label{app:yearprediction-results}

We compare FlowDGP with DKL, NODE and NSDE on YearPredictionMSD
(Appendix~\ref{app:yearprediction}). Table~\ref{tab:yearprediction} reports
the results: FlowDGP achieves the lowest mean RMSE and NLL, while NSDE
achieves the lowest mean CRPS.

\begin{table}[htbp]
\centering
\caption{YearPredictionMSD results, mean $\pm$ standard deviation over
random seeds. Lower is better; bold marks the lowest mean.}
\label{tab:yearprediction}
\small
\begin{tabular}{lccc}
\toprule
Method & RMSE & NLL & CRPS\\
\midrule
DKL & $0.8090 \pm 0.0022$ & $1.2078 \pm 0.0023$ & $0.4243 \pm 0.0012$\\
NODE & $0.8096 \pm 0.0018$ & $1.2088 \pm 0.0025$ & $0.4252 \pm 0.0018$\\
NSDE & $0.8075 \pm 0.0022$ & $1.2074 \pm 0.0033$ & $\mathbf{0.4227 \pm 0.0016}$\\
FlowDGP & $\mathbf{0.8052 \pm 0.0015}$ & $\mathbf{1.2027 \pm 0.0023}$ & $0.4234 \pm 0.0016$\\
\bottomrule
\end{tabular}
\end{table}

\subsection{A controlled contraction illustration}
\label{app:contraction-diagnostic}

To illustrate Corollary~\ref{cor:contractive_stability}, we consider
an auxiliary one-dimensional SDE, separate from the fitted CDGP and
FlowDGP models. Paired paths share the same Brownian motion, and we measure
\[
S(T)=\frac{|h_T(\theta_2)-h_T(\theta_1)|}{|\theta_2-\theta_1|}.
\]
The contractive drift is $f(h;\theta)=-\mu h-0.5\tanh(h)+\theta$;
the control is $f(h;\theta)=0.25h+\theta$.
For the contractive drift,
$\partial_h f=-\mu-0.5\,\mathrm{sech}^2(h)\le-\mu$ and $L_\theta=1$.
With equal initial conditions and shared parameter-independent additive
noise, the corollary bounds sensitivity by $1/\mu$.
Table~\ref{tab:contraction-diagnostic} reports the results from
5,000 paired paths. Sensitivity levels off below the bound in the
contractive cases, while the control increases with the horizon.

\begin{table}[htbp]
\centering
\caption{Controlled SDE illustration of Corollary~\ref{cor:contractive_stability}.}
\label{tab:contraction-diagnostic}
\begin{tabular}{lccccc}
\toprule
Drift & $\mu$ & $S(1)$ & $S(4)$ & $S(8)$ & $L_\theta/\mu$\\
\midrule
Contractive & 0.5 & 0.637 & 1.002 & 1.023 & 2.000\\
Contractive & 1.0 & 0.521 & 0.672 & 0.673 & 1.000\\
Contractive & 2.0 & 0.368 & 0.401 & 0.401 & 0.500\\
Non-contractive & --- & 1.136 & 6.870 & 25.542 & ---\\
\bottomrule
\end{tabular}
\end{table}

\clearpage
\subsection{Learned drift evolution}
\label{app:drift-visualization}

We use the toy dataset to visualize how the learned drift and latent
state evolve with continuous depth in CDGP and FlowDGP.

Both models use $T=1$, $J=3$, 30 inducing points for each state or base GP,
and 2,500 training epochs; CDGP uses five depth inducing points.
The FlowDGP drift posterior mean is estimated using 512 base-GP samples.
We plot eight latent trajectories for each of seven inputs.

Figure~\ref{fig:toy-drift} plots the drift posterior means at fixed
states for CDGP and fixed inputs for FlowDGP, together with sampled
latent trajectories. In this toy example, CDGP's drift means change
modestly with $t$, whereas FlowDGP's vary more substantially.
The latent trajectories show a similar pattern, with more pronounced
changes across depth for FlowDGP.

\begin{figure}[htbp]
\centering
\includegraphics[width=\linewidth]{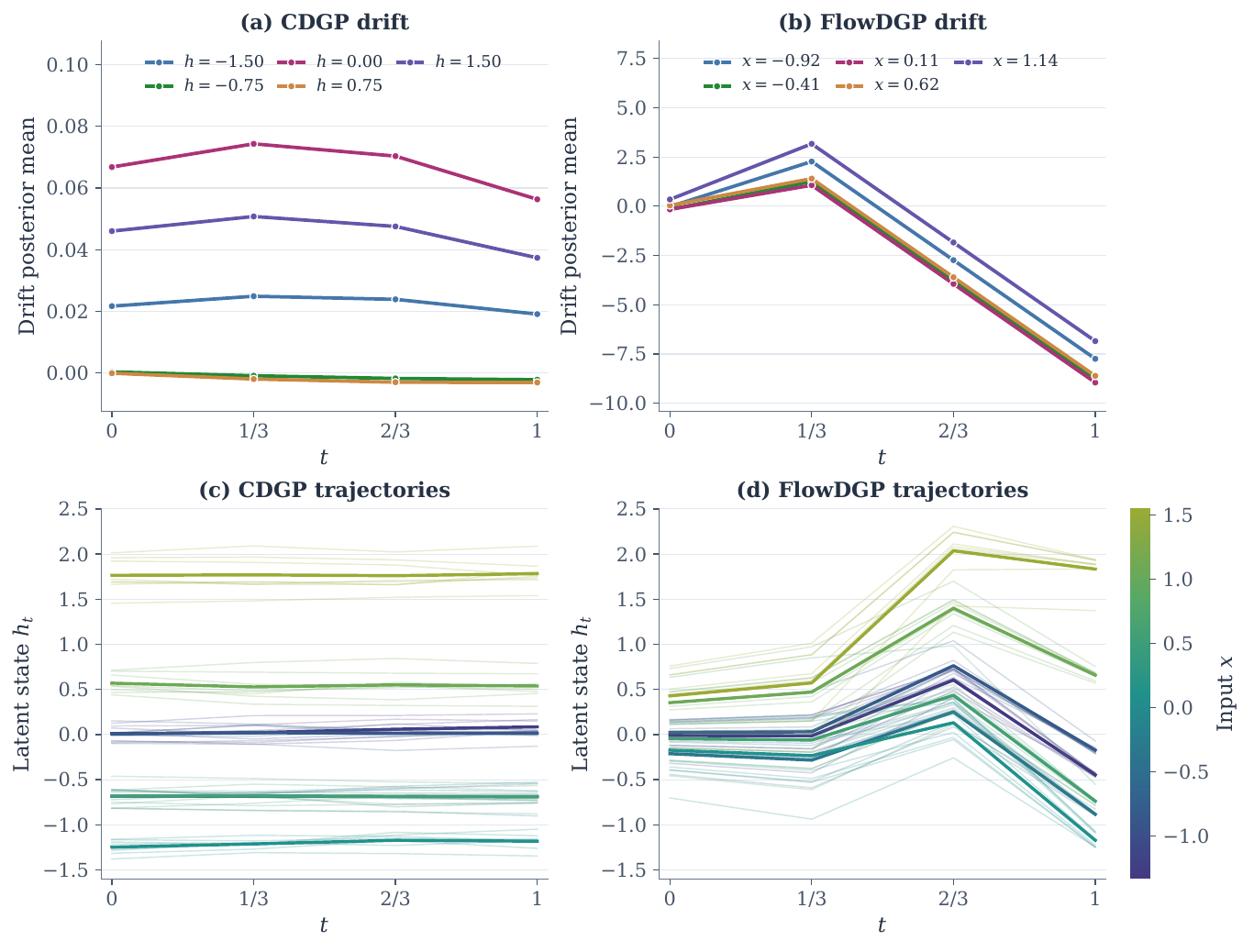}
\caption{\textbf{Learned drift evolution and latent trajectories on the toy dataset.}
Top: drift posterior means versus $t$ at fixed states $h$ for CDGP
and fixed inputs $x$ for FlowDGP. Bottom: sampled latent trajectories
(thin lines) and their means (thick lines), colored by input.}
\label{fig:toy-drift}
\end{figure}
\clearpage

\section{LLM Usage}

Large language models were used solely for editing and polishing the
writing of this paper.

\end{document}